\documentclass[12pt,nosubfloats]{colt2024} 

\title[]{Expressivity of Spiking Neural Networks}

\usepackage{times}
\coltauthor{\Name{Manjot Singh$^*$} \Email{singh@math.lmu.de}\\
 \Name{Adalbert Fono$^*$} \Email{fono@math.lmu.de}\\
 \Name{Gitta Kutyniok} \Email{kutyniok@math.lmu.de}\\
 \addr Ludwig-Maximilians-Universit\"at M\"unchen}

\usepackage{lipsum}
\usepackage{listings}

\usepackage{tikz}
\usetikzlibrary{decorations.pathreplacing}
\usetikzlibrary{fadings}
\usepackage{caption}
\usepackage{subcaption}

\usepackage{algorithm} 
\usepackage{soul}
\usepackage{algpseudocode} 
\newenvironment{hproof}{%
\renewcommand{\proofname}{Sketch of proof}\proof}{\endproof}

\DeclareMathSymbol{\shortminus}{\mathbin}{AMSa}{"39}

\providecommand{\R}{\mathbb{R}} 
\providecommand{\N}{\mathbb{N}} 

\providecommand{\cR}{\mathcal{R}}

\providecommand{\abs}[1]{\left\lvert#1\right\rvert}
\providecommand{\norm}[1]{\left\lVert#1\right\rVert}

\begin{document}

\maketitle
\def\thefootnote{$*$}\footnotetext{Equal Contribution}

\begin{abstract}%
The synergy between spiking neural networks and neuromorphic hardware holds promise for the development of energy-efficient AI applications. Inspired by this potential, we revisit the foundational aspects to study the capabilities of spiking neural networks where information is encoded in the firing time of neurons.
Under the Spike Response Model as a mathematical model of a spiking neuron with a linear response function, we compare the expressive power of artificial and spiking neural networks, where we initially show that they realize piecewise linear mappings. In contrast to ReLU networks, we prove that spiking neural networks can realize both continuous and discontinuous functions. Moreover, we provide complexity bounds on the size of spiking neural networks to emulate multi-layer (ReLU) neural networks. Restricting to the continuous setting, we also establish complexity bounds in the reverse direction for one-layer spiking neural networks.%
\end{abstract}

\begin{keywords}%
  Expressivity, Approximation Theory, Spiking Neural Networks, Deep (ReLU) Neural Networks, Temporal Coding, Linear Regions
\end{keywords}

\section{Introduction}
Spiking Neural Networks (SNNs), sometimes considered as the third generation of neural networks, have recently emerged as a notable paradigm in neural computing. In traditional artificial neural networks (ANNs), information is propagated synchronously through the network, whereas SNNs are based on asynchronous information transmission in the form of an \textit{action-potential} or a \textit{spike} \citep{gerstner_kistler_naud_paninski_2014}. Spikes can be considered as point-like events in time, where incoming spikes received via a neuron's synapses trigger new spikes in the outgoing synapses. Hence, a key difference between ANNs and SNNs lies in the significance of timing in the operation of SNNs. Moreover, the (typically) real-valued input information to an SNN needs to be encoded in the form of spikes, necessitating a spike-based encoding scheme. 

Different encoding schemes enable spiking neurons to represent real-valued inputs, broadly categorized into rate and temporal coding \citep{gerstner_1993_assemblies}.
Rate coding refers to the number of spikes in a given time period whereas in temporal coding, the precise timing of spikes matters \citep{MAASS2001ontherelevanceoftime}. The notion of firing rate adheres to neurobiological experiments where it was observed that some neurons fire frequently in response to some external stimuli \citep{Stein1967TheIC, gerstner_kistler_naud_paninski_2014}. However, the firing rate results in high latency and is computationally expensive due to an overhead related to temporal averaging. The latest experimental results indicate that the firing time of a neuron is essential for the system to respond fast to more complex sensory stimuli \citep{Hopfield1995inputtiming, thorpe1996150ms, abeles_1991}. With the firing time, each spike carries a significant amount of information, thus the resulting signal can be quite sparse.

While there is no general consensus on the description of neural coding, in this work, we assume that the information is encoded exclusively in the firing time of a neuron. 
The event-driven, sparse information propagation, as seen in time-to-first-spike encoding \citep{Gerstner2002TTFS}, facilitates system efficiency in terms of reduced computational power and improved energy efficiency in comparison to the substantial time and energy consumption associated with training and inferring on ANNs \citep{Thompson2021DimReturns}. This concept is particularly relevant in the context of neuromorphic computing \citep{neuromprohic2022}, where a hardware architecture based on SNNs is designed to mimic the human brain's structure and functioning to achieve efficient information processing.

It is clear that the differences in the processing of information between ANNs and SNNs should also lead to differences in the computations performed by these models. 
Several groups have analyzed the expressive power of ANNs \citep{yarotsky2016error_bounds, Cybenko1989ApproximationBS, kutyniok2020nnaproxexpressvity, PETERSEN_2018_optimalapprox}, and in particular provided explanations for the superior performance of deep networks over shallow ones \citep{daub_2019_NNapprox, yarotsky2016error_bounds}. In the case of ANNs with ReLU activation function, the number of linear regions into which the input space is partitioned is another property that highlights the advantages of deep networks over shallow ones. Unlike shallow networks, deep networks divide the input space into exponentially more linear regions \citep{linearregions_unser_2022, montafur2014linearregions} enabling them to express more complex functions. There exist further approaches to characterize the expressiveness of ANNs, e.g., the concept of VC-dimension in the context of classification problems \citep{bartlett1998VC, vc_goldberg_1995, JMLRvcbartlett}.

Few attempts have been made to understand the computational power of SNNs. The works by \citet{Maass1996NoisySN, Maass1996lowerbounds} demonstrate the capability of spiking neurons to emulate Turing machines, arbitrary threshold circuits, and sigmoidal neurons in temporal coding. In \citet{Maass1996Networksthirdgeneration}, biologically relevant functions are depicted that can be emulated by a single spiking neuron but require complex ANNs to achieve the same task. \cite{temporal_single_spike_backprop2020_comsa}, \cite{Maass95approxsigmoid} showed that continuous functions can be approximated to arbitrary precision in temporal coding.
A connection between SNNs and piecewise linear functions was noted in \citet{mostafa2018pwl}. The author showed that an SNN consisting of non-leaky integrate and fire neurons and temporal coding exhibits a piecewise linear input-output relation after a transformation of the time variable.
A common theme is that the model of spiking neurons and the description of their dynamics varies, i.e., they are chosen and adjusted for a specific goal or task. 
\citet{optimizedFS_maass2021} aims at generating high-performance SNNs for image classification using a modified spiking neuron model that limits the number of spikes emitted by each neuron while considering precise spike timing. In \citet{zhang2022firingratesapprox}, the authors investigate self-connection SNNs, demonstrating their capacity to efficiently approximate discrete dynamical systems. 
\citet{moraitis2021optimality} showcases the ability of SNNs using short-term spike-timing-dependent-plasticity mechanism to model certain dynamic environments.

The primary challenge in advancing the field of SNNs has revolved around devising training methodologies. The typical approach is to either train from scratch \citep{approx_grad_lee_2020, approx_grad_wu_2018, temporal_single_spike_backprop2020_comsa, firstspike2021goltz} or convert trained ANNs into SNNs performing the same tasks \citep{bodo2017cnnsnn, Kim2018DeepNN, bodo2021temporalpatterncoding,  bodo2018convertinganntosnn, approxrelu2022ana, relutosnn2022ana, Zhang_Zhou_Zhi_Du_Chen_2019, Yan_Zhou_Wong_2021}. 
The latter works concentrate on the algorithmic construction of SNNs approximating or emulating given ANNs for various spike patterns, encoding schemes, and spiking neuron models. Therefore, we aim to extend the theoretical understanding of the differences and similarities in the expressive power between a network of spiking and artificial neurons employing a piecewise-linear activation function.

\paragraph{Contributions} In this paper, to analyze SNNs, we employ the noise-free version of the Spike Response Model (SRM) \citep{timestructure1995gerstner}. It describes the state of a neuron as a weighted sum of response and threshold functions. We assume a linear response function, where additionally each neuron spikes at most once to encode information through precise spike timing. The spiking networks based on the linear SRM are succinctly referred to as LSNNs. 
The main results are centered around the comparison of expressive power between LSNNs and ANNs:

\begin{itemize}
    \item[$\bullet$] \textbf{Similarities between LSNNs and ReLU-ANNs:} We show that LSNNs are as expressive as ANNs with piecewise activation when expressing various functions. 
\begin{itemize}
    \item[\scriptsize$\bullet$] We prove that the mapping generated by an LSNN is piecewise linear and under certain settings continuous, concave, and increasing. 
    \item[\scriptsize$\bullet$] We show that there exists an LSNN that emulates the ReLU non-linearity. Then, we extend the result to multi-layer neural networks and show that LSNNs have the capacity to effectively emulate any (ReLU) ANN. Furthermore, we present explicit complexity bounds for constructing an LSNN capable of realizing an equivalent ANN. We also provide insights into the influence of the encoding scheme and the impact of different parameters on the above expressivity results. These findings imply that LSNNs can approximate any function as accurately as deep ANNs with a piecewise linear activation function.
\end{itemize}
\end{itemize}

\begin{itemize}
    \item[$\bullet$] \textbf{Differences between LSNNs and ReLU-ANNs:} We prove distinctive characteristics of LSNNs that distinguish them from ReLU-ANNs, thus illustrating differences in the structure of computations between LSNNs and ANNs.

\begin{itemize}
    \item[\scriptsize$\bullet$] We show that the mapping generated by LSNNs may be discontinuous which is in contrast to a ReLU-ANN, which outputs a continuous piecewise linear mapping. This suggests that LSNNs might be better suited for approximating / realizing discontinuous piecewise functions. 
    \item[\scriptsize$\bullet$] We demonstrate that the maximum number of linear regions that a one-layer LSNN generates scales exponentially with input dimension. Consequently, a shallow LSNN can be as expressive as a deep ReLU network in terms of the number of linear regions required to express certain types of continuous piecewise linear functions. Additionally, we give upper bounds on the size of ReLU-ANNs to emulate one-layer LSNNs. 
\end{itemize}
\end{itemize}

\paragraph{Broader impact} 
The findings presented herein deepen our understanding of the theoretical capabilities of SNNs and their differences from ANNs. Although we consider a simplified model of spiking dynamics within the LSNN framework, we obtain insights into event-driven computations where time plays a critical role.
Moreover, our results further extend the understanding of the approximation properties of spiking neural networks, emphasizing their potential as an alternative computational model for handling complex tasks. 
By studying the theoretical power of SNNs, we aim to contribute to the realization of energy-efficient and low-power AI on neuromorphic hardware, providing viable options in contrast to established deep learning models.

\paragraph{Outline} In Section \ref{section:SNN_model}, we introduce necessary definitions, including spiking neural networks and their realization under the Spike Response Model. We present our main results in Section \ref{section: main_results}. 
We conclude in Section \ref{section:discussion} by summarizing the limitations and implications of our results. The proofs of all the results are provided in Appendix \ref{section:Appendix}.

\section{Spiking neural networks}
\label{section:SNN_model}
In neuroscience literature, several mathematical models exist that describe the generation and propagation of action-potentials. 
Action-potentials or spikes are short electrical pulses that are the result of electrical and biochemical properties of a biological neuron. We refer to \citet{gerstner_kistler_naud_paninski_2014} for a comprehensive and detailed introduction to the dynamics of spiking neurons. To study the expressivity of SNNs, the main principles of a spiking neuron are condensed into a (simplified) mathematical model, where certain details about the biophysics of a biological neuron are neglected.

\subsection{Spiking neurons under Spike Response Model}
Following \citet{Maass1996NoisySN}, we consider the Spike Response Model (SRM) \citep{timestructure1995gerstner} as a formal model for a spiking neuron. It effectively captures the dynamics of the Hodgkin-Huxley model \citep{gerstner1997volterra, gerstner_kistler_naud_paninski_2014}, the most accurate model in describing neuronal dynamics, and is a generalized version of the leaky integrate and fire model \citep{timestructure1995gerstner}. The SRM leads to the subsequent definition of an SNN \citep{Maass1996lowerbounds}. 

\begin{definition}
    \label{def:SNN_General}
    A \emph{spiking neural network} $\Phi$ under the SRM is a (simple) finite directed graph $(V,E)$ and consists of a finite set $V$ of spiking neurons, a subset $V_{\text{in}} \subset V$ of input neurons, a subset $V_{\text{out}} \subset V$ of output neurons, and a set $E \subset V \times V$ of synapses. Each \emph{synapse} $(u, v) \in E$ is associated with a \emph{synaptic weight} $w_{uv} \geq 0$, a \emph{synaptic delay} $d_{uv} \geq 0$, and a \emph{response function} $\varepsilon_{uv} : \R \to \R$, which depends on the synaptic delay. Each neuron $v \in V \setminus  V_{\text{in}}$ is associated with a \emph{firing threshold} $\theta_v > 0$, and a \emph{membrane potential} $P_v: \R \to \R$, which is given by
    \begin{equation}\label{eqn:SRM}
        P_v(t) = \sum_{(u,v)\in E} \sum_{t_u^f \in F_u} w_{uv}\varepsilon_{uv}(t - t_u^f),
    \end{equation}
    where $F_u = \{t_u^f: 1\leq f \leq n \text{ for some } n\in\N\}$ denotes the set of firing times of a neuron $u$, i.e., times $t$ whenever $P_u(t)$ reaches $\theta_u$ from below. 
\end{definition}

In general, the membrane potential also includes the \textit{threshold function} $\Theta_v: \R_{\geq0} \to \R_{>0}$, that models the refractoriness effect. That is, if a neuron $v$ emits a spike at time $t_v^f$, $v$ cannot fire again for some time interval immediately after $t_v^f$, regardless of how large its potential might be. However, we assume that each neuron fires at most once, i.e., information is encoded in the firing time of single spikes. Thus, in Definition \ref{def:SNN_General}, the refractoriness effect can be ignored, and the contribution of $\Theta_v$ is modeled by the constant $\theta_v$. 
Moreover, the single spike condition simplifies \eqref{eqn:SRM} to
\begin{equation}\label{eq:SRM2}
    P_v(t) = \sum_{(u,v)\in E} w_{uv}\varepsilon_{uv}(t - t_u), \quad \text{ where } t_u \text{ is the firing time of presynaptic neuron } u.
\end{equation}
The response function $\varepsilon_{uv}$ models the impact of a spike from a presynaptic neuron $u$ on the membrane potential of a postsynaptic neuron $v$ \citep{timestructure1995gerstner}.
A biologically realistic approximation of $\varepsilon_{uv}$ is a delayed $\alpha$ function \citep{timestructure1995gerstner}, which is non-linear and leads to intractable problems when analyzing the propagation of spikes through an SNN. Hence, following \citet{Maass1996NoisySN}, we consider a simplified response and only require $\varepsilon_{uv}$ to satisfy the following condition:
\begin{equation}\label{eq:Response_cond}
    \varepsilon_{uv}(t) =  \begin{cases}
                                            0, & \text{if } t \notin [d_{uv}, d_{uv} + \delta],\\
                                            s \cdot (t-d_{uv}),  & \text{if } t \in [d_{uv}, d_{uv} + \delta],
    \end{cases}  
     \quad \text{ where } s\in\{+1,-1\} \text{ and } \delta > 0.
\end{equation}
The parameter $\delta$ is some constant assumed to be the length of a linear segment of the response function. The variable $s$ reflects the fact that biological synapses are either \textit{excitatory} or \textit{inhibitory} and the synaptic delay $d_{uv}$ is the time required for a spike to travel from $u$ to $v$. Inserting condition \eqref{eq:Response_cond} in \eqref{eq:SRM2} and setting $w_{uv} := s\cdot w_{uv}$, i.e., allowing $w_{uv}$ to take arbitrary values in $\R$, yields
        \begin{equation}\label{eq:SRM3}
            P_v(t) =\sum_{(u,v)\in E} \mathbf{1}_{\{0 < t - t_u - d_{uv} \leq \delta\}}  w_{uv} (t - t_u -d_{uv})  
        \end{equation}

Using \eqref{eq:SRM3} enables us to iteratively compute the firing time $t_v$ of each neuron $v \in V\setminus V_{\text{in}}$ if we know the firing time $t_u$ of each neuron $u\in V$ with $(u,v)\in E$ by solving for $t$ in 
\begin{align}
    \inf_{t\in \R}  P_v(t) = \inf_{t\in \R}   \sum_{(u,v)\in E}  \mathbf{1}_{\{0 < t - t_u - d_{uv} \leq \delta\}} w_{uv} (t - t_u -d_{uv}) = \theta_v, \label{eq:potential}\\  
    \text{ i.e., } \quad t_v = \frac{\theta_v + \sum_{(u,v)\in E}  \mathbf{1}_{\{0 < t_v - t_u - d_{uv} \leq \delta\}} w_{uv} (t_u + d_{uv})}{\sum_{(u,v)\in E} \mathbf{1}_{\{0 < t_v - t_u - d_{uv}  \leq \delta\}} w_{uv} }. \nonumber
\end{align}

Observe that $t_v$ is a weighted sum (up to a positive constant) of the firing times of neurons $u$, $(u,v)\in E$, actually contributing to the firing of $v$. For instance, if $t_z + d_{zv} > t_v$ for some synapse $(z,v)\in E$, then $z$ did not influence the firing of $v$ since the spike from $z$ arrived after $v$ already fired. Depending on the firing time of the presynaptic neurons and the associated parameters (weights, delays, threshold), a specific subset of presynaptic neurons triggers the firing in $v$ so that $t_v$ changes accordingly. The dynamics of a neuron in this model are depicted in Figure \ref{fig: DynamicsSRM}.

\begin{definition}
\label{rem:LSNN}    
    We call an SNN based on the SRM with the additional assumptions \eqref{eq:Response_cond} and \eqref{eq:SRM3} an \emph{LSNN} and the corresponding spiking neurons \emph{LSNN neurons}.
\end{definition}

\begin{figure}[t]
    \centering
    \begin{subfigure}[b]{0.4\textwidth}
    \hspace{1cm}
    	\begin{tikzpicture}[shorten >=1pt, transform canvas = {scale=0.75}]
		\tikzstyle{unit}=[draw,shape=circle,minimum size=0.5cm]
		\tikzstyle{hidden}=[draw,shape=circle,fill=black!25,minimum size=0.5cm]
		\tikzstyle{hidden}=[draw,shape=circle,minimum size=0.5cm]
		\node[unit](x0) at (0,3.5){$t_{u_1}$};
            \node[unit](x1) at (2,3.5){$t_{u_2}$};
		\node at (3.4,3.5){\dots\dots};
		\node[unit](xd) at (5,3.5){$t_{u_5}$};
		\node[unit](h11) at (3,6){$t_v$}; 
		\draw[->] (x0) -- (h11);
		\draw[->] (x1) -- (h11);
		\draw[->] (xd) -- (h11);
            \draw[->, scale=1, >=stealth, line width=1pt, bend right] (4,6) to (6.5,6);
            \draw [decorate,decoration={brace,amplitude=10pt,mirror}](-0.6,2.7) -- (5.7,2.7) node[midway,yshift=-1.7em]{input neurons};
		\draw [decorate,decoration={brace,amplitude=10pt},xshift=-4pt,yshift=0pt] (2.5,6.3) -- (3.75,6.3) node [black,midway,yshift=+0.6cm]{output neuron};
	\end{tikzpicture}
	\caption{}
        \label{subfig:snn_firing_time}
    \end{subfigure}
    \begin{subfigure}[b]{0.59\textwidth}
    \centering 
    \includegraphics[height = 6.7cm, width=\linewidth]{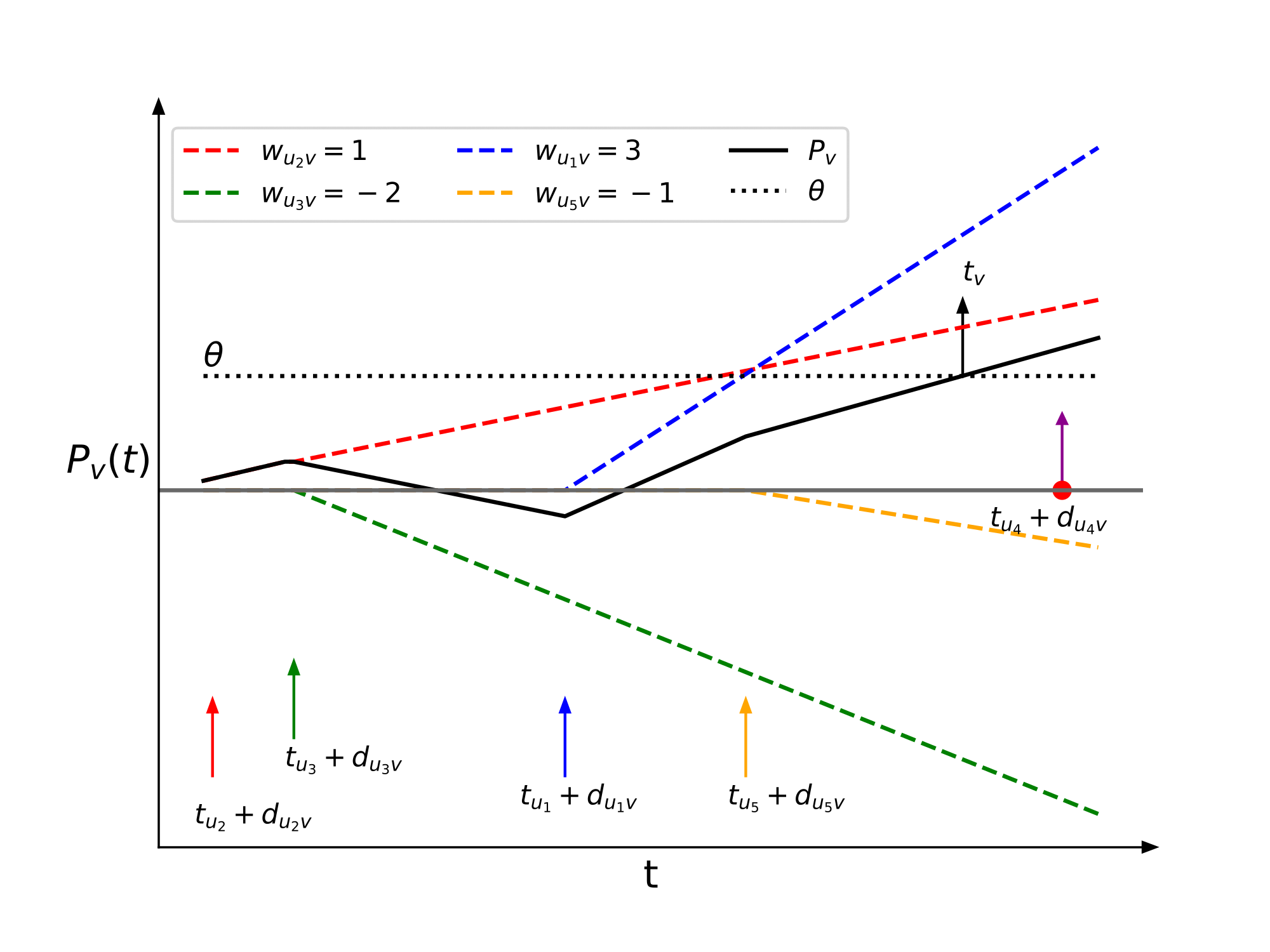}
    \caption{}
    \label{subfig:snn_evolution_of_potential}
    \end{subfigure}
    \caption{(a) An LSNN neuron $v$ with five input neurons $u_1,\dots,u_5$ that fire at times $t_{u_1},\dots,t_{u_5}$, respectively. (b) The trajectory in black shows the evolution of the membrane potential $P_v(t)$ of $v$ as a result of incoming spikes (vertical arrows). Neurons $u_1$ and $u_2$ generate positive responses, whereas neurons $u_3$ and $u_5$ trigger negative responses, with the response magnitudes denoted by $w_{u_iv}$. The spike from neuron $u_4$ does not influence the firing time $t_v$ of $v$ since $t_v < t_{u_4} + d_{u_4v}$.}
    \label{fig: DynamicsSRM} 
\end{figure}
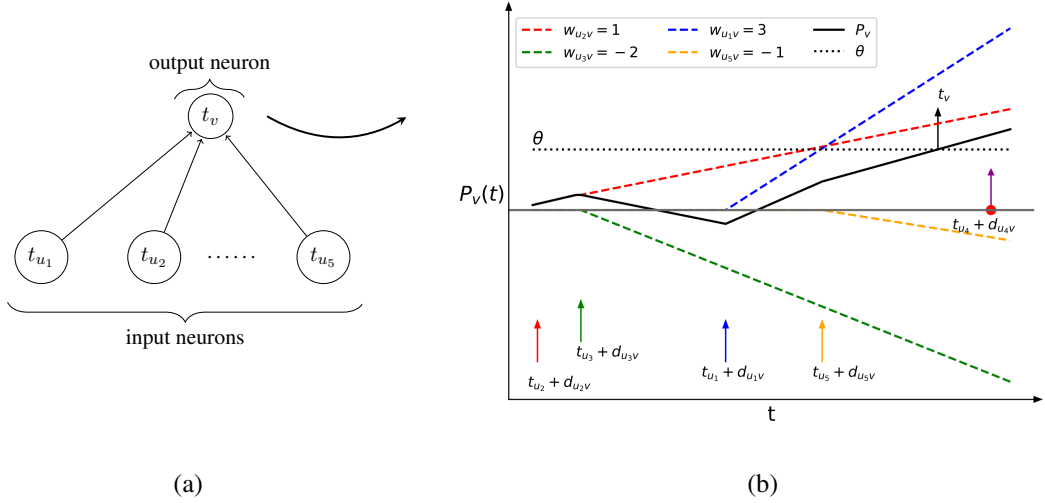

\subsection{Realizations of LSNNs}
\label{subsec:realization_SNNs_defn}

A common representation of feedforward ANNs is based on a sequence of matrix-vector tuples \citep{Berner_2022}, \citep{PETERSEN_2018_optimalapprox}, whereby a distinction between the network and the target function it realizes is established.
\begin{definition}
\label{def:ANN_realization}  
Let $L, N_0, \dots, N_L \in \N$. An \emph{artificial neural network} $\Psi$ is a sequence of matrix-vector tuples
\begin{equation*}
    \Psi = ((W^1, B^1), (W^2, B^2), \dots, (W^L, B^L)),
\end{equation*}
where each $W^\ell \in \R^{N_{\ell-1}\times N_\ell}$ and $B^\ell \in \mathbb{R}^{N_\ell}$. $N_0$ and $N_L$ are the input and output dimension of $\Psi$. 
We call $N(\Psi) := \sum_{j=0}^{L}N_j$ \emph{the number of neurons of the network} $\Psi$, $L(\Psi) := L$ \emph{the number of layers} of $\Psi$ and $N_\ell$ \emph{the width of $\Psi$} in layer $\ell$.    
The \emph{realization} of $\Psi$ with component-wise \emph{activation function} $\sigma:\R \to\R$ is defined as the map $\cR_{\Psi}:\R^{N_0} \rightarrow \R^{N_L}$, $\cR_{\Psi}(x) = y_L$, where $y_L$ results from  
\begin{align}\label{eq:ANNcomp}
    y_0 &= x, \hspace{0.5em} y_\ell = \sigma((W^\ell)^T y_{\ell-1} + B^\ell), \:\: \text{for} \: \ell=1,\dots, L-1, \hspace{0.5em} \text{and} \:\: y_L = (W^L)^T y_{L-1} + B^L.
\end{align}
\end{definition}
\begin{remark}
     Henceforth, $\sigma(x) = \max(0,x)$ denotes the \emph{ReLU activation}.
\end{remark}
An analogous framework can be derived for LSNNs by arranging the underlying graph in layers and equivalently representing LSNNs by a sequence of their parameters. 
\begin{definition}
\label{def:SNN_matrix}
Let $L, N_0, \dots, N_L \in \N$. An \emph{LSNN} $\Phi$ associated to the acyclic graph $(V,E)$ is a sequence of matrix-matrix-vector tuples 
\begin{equation*}
    \Phi = ((W^1, D^1, \Theta^1), (W^2, D^2, \Theta^2), \dots, (W^L, D^L, \Theta^L))
\end{equation*}
where each $W^l \in \mathbb{R}^{N_{\ell-1}\times N_\ell}$, $D^\ell \in \mathbb{R}_{\geq0}^{N_{\ell-1}\times N_\ell}$, and $\Theta^\ell \in \mathbb{R}_{>0}^{N_\ell}$. The matrix entries $W^\ell_{uv}$ and $D^\ell_{uv}$ represent the weight and delay value associated with the synapse $(u,v)\in E$, respectively, and the entry $\Theta^\ell_v$ is the firing threshold associated with node $v \in V$ in layer $\ell$. $N_0$ is the input dimension and $N_L$ is the output dimension of $\Phi$. We call $N(\Phi) := \sum_{j=0}^{L}N_j$ the \emph{number of neurons} and $L(\Phi) := L$ denotes the \emph{number of layers} of $\Phi$.
\end{definition}
Before turning to the realization of LSNNs, we highlight two assumptions we will rely on, which allow us to analyze the LSNN framework in the most basic setting: 

\textbf{Assumption I}: The parameter $\delta$ describing the length of the linear segment of the response function introduced in \eqref{eq:Response_cond} is assumed to be large or even infinite. Then the minimization problem in \eqref{eq:potential} to obtain the firing time $t_v$ of a neuron $v$ simplifies to
\begin{equation}\label{eqn:firing_time}
    \inf_{t\in\R}  P_v(t) =  \inf_{t\in \R}   \sum_{(u,v)\in E}  \mathbf{1}_{\{t> t_u + d_{uv}\}} w_{uv} (t - t_u -d_{uv}) = \inf_{t\in \R}   \sum_{(u,v)\in E} w_{uv} \sigma(t - t_u -d_{uv}) = \theta_v 
\end{equation}
so that 
\begin{equation}\label{eqn:firing_time_specific}
     t_v = \frac{\theta_v + \sum_{(u,v)\in E : t_v > t_u + d_{uv}} w_{uv} (t_u + d_{uv})}{\sum_{(u,v)\in E : t_v > t_u + d_{uv}}  w_{uv} }.    
\end{equation}
Informally, a large linear segment entails that spikes have a constant effect on postsynaptic neurons so that spikes do not act as point-like events in time. The obtained framework, which exhibits similarities to the integrate and fire model, enables us to assess the spiking dynamics of LSNN neurons and gain insights into what we can expect from more generalized models incorporating multi-spike responses and refractoriness effects. In contrast, a biologically more realistic small linear segment requires incoming spikes to have a correspondingly small time delay to jointly affect the potential of a neuron. Otherwise, the impact of the earlier spikes on the potential may already have vanished before the subsequent spikes arrive. The resulting model resembles the leaky integrate and fire model. In conclusion, incorporating $\delta$ as an additional parameter in the LSNN framework leads to additional complexity since the same firing patterns may result in different outcomes. However, an in-depth analysis of this effect is left as future work.

\textbf{Assumption II}: The sum of incoming weights of each neuron $v$ in an LSNN $\Phi$ is assumed to be positive. The positivity ensures that each neuron in $\Phi$ emits a spike, in particular, it is a sufficient but not a necessary condition to guarantee that spikes are emitted by the output neurons. One can certainly treat LSNNs without requiring that neurons have to spike, which again leads to augmented complexity due to increased flexibility in the model.

Under the introduced conditions, the firing time of LSNN neurons can be considered as well-defined mappings in the following sense.
\begin{definition}\label{def:firingMap}
    Let $\Phi$ be an LSNN with input neurons $u_1, \dots, u_d$ and output neurons $v_1, \dots, v_n$. For any firing time of the input neurons $(t_{u_1}, \dots, t_{u_d})^T \in \R^d$ and the corresponding firing times of the output neurons $(t_{v_1},\dots,t_{v_n})^T \in \R^n$ determined via \eqref{eqn:firing_time}, we denote by $t_\Phi: \R^d \to \R^n$, $(t_{u_1}, \dots, t_{u_d}) \mapsto t_\Phi(t_{u_1}, \dots, t_{u_d})=(t_{v_1},\dots,t_{v_n})^T$ the \emph{firing mapping} of $\Phi$.
\end{definition}
The key feature of any SNN is the asynchronous information propagation in the spiking domain due to variable firing times among neurons.
Hence, to employ SNNs, the (typically real-valued) input information needs to be encoded in the firing times of the neurons in the input layer, and similarly, the firing times of the output neurons need to be translated back to an appropriate target domain. We will refer to this process as input encoding and output decoding. The applied encoding scheme certainly depends on the specific task at hand and the potential power and suitability of different encoding schemes is a topic that warrants separate investigation on its own. 
Our focus in this work lies on exploring the intrinsic capabilities of LSNNs, rather than the specifics of the encoding scheme.

Thus, we can formulate some guiding principles for establishing a reasonable encoding scheme. First, the firing times of input and output neurons should encode real-valued information in a consistent way so that different networks can be concatenated in a well-defined manner. This enables us to construct suitable subnetworks and combine them appropriately to solve more complex tasks. One can perform basic actions on neural networks such as concatenation and parallelization to construct larger networks from existing ones. 
Adapting a general approach for ANNs as defined in \citet{Berner_2022, PETERSEN_2018_optimalapprox}, we formally introduce the concatenation and parallelization of networks of spiking neurons in the Appendix \ref{section:SNN_calculus}. 
Second, in the extreme case, the encoding scheme might directly contain the solution to a problem, underscoring the need for a sufficiently simple and broadly applicable encoding scheme to avoid this.

\begin{definition}\label{definition:encoding}
    Let $[a,b]^d \subset \R^d$ and $\Phi$ be an LSNN with input neurons $u_1, \dots, u_d$ and $n$ output neurons. 
    Fix reference times $T_{\text{in}}\in\R^d$ and $T_{\text{out}} \in \R^n$ via $T_{\text{in}}= t_{\text{in}}\, (1,\dots,1)^T$ and $T_{\text{out}}= t_{\text{out}}\, (1,\dots,1)^T$, respectively, where 
    $t_{\text{in}},t_{\text{out}} \in \R$ with $t_{\text{out}} > t_{\text{in}}$. For any $x \in [a,b]^d$, we set the firing times of the input neurons to $(t_{u_1}, \dots, t_{u_d})^T = T_{\text{in}} + x$. The corresponding firing times of the output neurons $t_\Phi(t_{u_1}, \dots, t_{u_d}) = T_{\text{out}} + y$ 
    encode the target $y \in \R^n$.
    The \emph{realization of} $\Phi$ is defined as the map $\mathcal{R}_{\Phi}: \mathbb{R}^d \to \mathbb{R}^{n}$, 
    \begin{equation*}
         \mathcal{R}_{\Phi}(x) = -T_{\text{out}} + t_\Phi(t_{u_1}, \dots, t_{u_d}) = y.
    \end{equation*}
\end{definition}
\begin{remark}
\label{remark:encoding}
    A bounded input range ensures that appropriate reference times can be fixed. Note that the introduced encoding scheme translates real-valued information into input firing times in a continuous manner. Occasionally, we will point out the effect of adjusting the scheme. 
\end{remark}

\section{Main results}
\label{section: main_results}
Subsequently, we will employ the framework introduced in Section \ref{section:SNN_model} to analyze the properties of LSNNs. First, we prove that LSNNs generate \textbf{C}ontinuous \textbf{P}iece\textbf{w}ise \textbf{L}inear (CPWL) mappings under certain conditions on the weights. Next, we show that LSNNs can emulate the realization of any multi-layer ANN employing ReLU as an activation function. We analyze the number of linear regions generated by LSNNs and compare the arising pattern to the well-studied case of ReLU-ANNs. 
Lastly, our findings show that LSNNs can efficiently realize certain CPWL functions using fewer computational units and layers compared to ReLU-ANNs.
If not stated otherwise, the encoding scheme introduced in Definition \ref{definition:encoding} is applied and the results need to be understood concerning this specific encoding. 

\subsection{Characterization of functions expressed by LSNNs}
\label{subsection:CPWL_mapping}
A broad class of ANNs based on a wide range of activation functions such as ReLU generate CPWL mappings \citep{daub_fractals_2020, DeVore2020NNApproximation}. In other words, these ANNs partition the input domain into regions, the so-called linear regions, on which an affine function represents the neural network's realization. Analyzing the firing mapping introduced in Definition \ref{def:firingMap}, we find that LSNNs exhibit a similar behaviour although the continuity is not necessarily maintained. 
The proof of the statement can be found in Appendix \ref{sec:realizationSNNs_appendix}. 
\begin{theorem}\label{thm:CPWL_mapping}
Let $\Phi = ((W^1, D^1, \Theta^1), (W^2, D^2, \Theta^2), \dots, (W^L, D^L, \Theta^L))$ be an LSNN. The firing mapping $t_\Phi$ is PWL. If additionally 
\begin{equation}\label{eq:SNN_CPWL}
    W^\ell_{uv} + \sum_{z : W^\ell_{zv} \leq 0} W^\ell_{zv} >0 \quad \text{for all } \ell \text{ and } u,v \text{ with } W^\ell_{uv} > 0
\end{equation}
holds, i.e., each incoming positive synaptic weight of any neuron $v$ is larger than the absolute value of the sum of its incoming negative synaptic weights, then $t_\Phi$ is a CPWL mapping. 
\end{theorem} 
\begin{hproof}
    First, via \eqref{eqn:firing_time_specific} one can derive that the firing mapping of an LSNN neuron with arbitrarily many presynaptic neurons is PWL. Since $\Phi$ consists of LSNN neurons arranged in layers it immediately follows that the firing map of each layer is PWL. Thus, as a composition of PWL mappings  $t_\Phi$ itself is PWL. Moreover, if all weights in $\Phi$ are positive, then the continuity of $t_\Phi$ at the breakpoints of the linear regions can be directly verified. In contrast, negative weights can under certain circumstances create a plateau or decrease the potential, causing a delay in firing, hence, resulting in a (jump-)discontinuity in $t_\Phi$. However, this effect can be excluded via \eqref{eq:SNN_CPWL} so that $t_\Phi$ is continuous if \eqref{eq:SNN_CPWL} holds.
\end{hproof}
The condition given in \eqref{eq:SNN_CPWL} is sufficient but not necessary to generate CPWL mappings; a corresponding example is provided in Appendix \ref{sec:realizationSNNs_appendix}. Under stronger assumptions, we can further characterize the properties of LSNNs. The properties can be verified mainly by repeated application of \eqref{eqn:firing_time} and \eqref{eqn:firing_time_specific}; the detailed computations are presented in Appendix \ref{sec:realizationSNNs_appendix}.
\begin{proposition}\label{prop:SNNproperties}
    Let $\Phi$ be an LSNN with only positive weights. Then $t_\Phi$ is an increasing and concave function. Additionally, the firing time of a neuron $v$ in $\Phi$ with corresponding 
    parameter $(w, d, \theta) \in \R^{d} \times \R_{\geq0}^{d} \times \R_{>0}$ and firing times $t_{u_1},\dots,t_{u_d} \in \R$ in the previous layer is given by 
    \begin{equation*}
        t_v(t_{u_1},\dots,t_{u_d}) = \inf_{\emptyset \neq I \subset [d]} \Big\{ s^I = \frac{1}{\sum_{i \in I} w_i}\big(\theta + \sum_{i \in I}w_i (t_{u_i} + d_i)\big) : s^I > \max_{i\in I} t_{u_i}   \Big\}.
    \end{equation*}
\end{proposition}

The properties described in Proposition \ref{prop:SNNproperties} are in general not true if the positive weights assumption is dropped. An immediate follow-up question is if the above findings apply to the realization of LSNNs via input and output encoding on a bounded domain. It is immediate to verify that the realization of an LSNN $\Phi$ is CPWL, increasing or concave if the encoding scheme and $t_\Phi$ are CPWL, increasing or concave, respectively, which certainly holds for the encoding presented in Definition \ref{definition:encoding}. However, even if $t_\Phi$ is not CPWL, increasing or concave, the corresponding feature can still arise in the realization due to the bounded input domain as the constructions in the subsequent results indicate. 
In particular, we show that LSNNs can realize the ReLU activation and as a consequence any multi-layer ReLU ANN. For the proof, please refer to Sections \ref{sec:RealizingReLU_app} and \ref{sec:approx_relu_network_using_SNN} in the Appendix.

\begin{theorem}
    \label{thm:approx_SNN_from_ANN}
    Let $L, d \in \N$, $[a,b]^d \subset \R^d$ and let $\Psi$ be an arbitrary ANN of depth $L$ and fixed width $d$ employing a ReLU non-linearity. 
    Then, there exists an LSNN $\Phi$ with $N(\Phi) = N(\Psi) + L(2d +3) - (2d+2)$ and $L(\Phi) = 3L-2$ that realizes $\mathcal{R}_\Psi$ on $[a,b]^d$.
\end{theorem}

\begin{hproof}
    Any multi-layer ANN with ReLU activation is simply an alternating composition of affine functions $A^\ell$ determined by the weights $W^\ell$ and biases $B^\ell$ in layer $\ell$ and a non-linear function represented by $\sigma$. Thus, to construct an LSNN $\Phi$ that realizes $\mathcal{R}_\Psi$, we first construct LSNNs that realize affine-linear functions and the ReLU non-linearity. Subsequently, we compose these subnetworks to obtain $\Phi$. Thereby, we realize $A^\ell$ by an LSNN with the same weights $W^\ell$, which may be negative. Therefore, in the LSNN construction, we rely on the value of the threshold parameter, which may depend on $[a,b]$, and auxiliary neurons with appropriate weights that ensure the firing of the output neurons and the desired (continuous) realization.
\end{hproof}

\begin{remark}
The result can be generalized to ANNs with varying widths that employ any type of PWL activation function. We note that the encoding scheme that converts the analog values into the time domain plays a crucial role. 
We construct a two-layer LSNN that realizes $\sigma$ via the encoding scheme  ($T_{\text{in}} + \cdot$) and ($T_{\text{out}} + \cdot$). At the same time, the encoding scheme ($T_{\text{in}} - \cdot$) and ($T_{\text{out}} - \cdot$) fails in the two-layer case, whereas utilizing an inconsistent input and output encoding enables us to construct a one-layer LSNN that realizes $\sigma$. This shows that not only the network but also the applied encoding scheme is highly relevant. For details, we refer to Appendix \ref{sec:RealizingReLU_app}.
\end{remark}

It is well known that ReLU-ANNs not only realize CPWL mappings but that every CPWL function can be represented by ReLU-ANNs \citep{arora_2016_understandingNN}. Theorem \ref{thm:approx_SNN_from_ANN} implies that LSNNs are as expressive as any ReLU-ANN, i.e., LSNNs can represent every ReLU-ANN and thereby every CPWL function with similar complexity. However, in a hypothetical real-world implementation, which certainly includes some noise, the constructed LSNN is not necessarily robust with respect to input perturbation.  Additionally, the complexity of an LSNN can be captured in other ways than in terms of the number of computational units and layers, e.g., the total number of spikes emitted in LSNNs is related to its energy consumption since emitting spikes consumes energy. Hence, the minimum number of spikes needed to realize a given function class may be a reasonable complexity measure with regard to energy efficiency for SNNs. Further research in these directions is necessary to analyze the behaviour under noise and provide error estimations as well as to evaluate the complexity of LSNNs via different measures with their benefits and drawbacks.

\subsection{Bounds on the complexity of ReLU-ANNs for expressing LSNNs}
\label{subsec:differences}
In this section, we further explore the differences in the computational structure between LSNNs and ReLU-ANNs. An already observed major distinction is the ability of LSNNs to realize discontinuous functions. Aside from this fact, can we establish dissimilarities when restricted to continuous realizations?
Since ReLU-ANNs can represent any CPWL mapping, they can realize any LSNN with a CPWL realization, in particular, LSNNs with positive weights and a CPWL encoding scheme. Hence, the key difference in the realization of arbitrary CPWL mappings may be the necessary size and complexity of the respective ANN and LSNN. To that end, we give upper bounds on the complexity of ReLU-ANNs needed to realize corresponding LSNNs. The first step in establishing the result is the study of the number of linear regions that LSNNs generate.

The number of linear regions can be seen as a measure of the flexibility and expressivity of the associated CPWL function. Similarly, we can measure the expressivity of an ANN by the number of linear regions of its realization. The connection of the depth, width, and activation function of an ANN to the maximum number of its linear regions is well-established, e.g., with increasing depth the number of linear regions can grow exponentially in the number of parameters of an ANN \citep{montafur2014linearregions,arora_2016_understandingNN,linearregions_unser_2022}. In the following, we observe a distinct scaling behaviour for LSNNs. 
\begin{lemma}\label{prop:Partition}
    Let $\Phi$ be a one-layer LSNN with a single output neuron $v$ and input neurons $u_1, \dots, u_d$. Then $t_\Phi$ partitions the input domain into at most $2^d-1$ linear regions. The maximal number of linear regions is attained if and only if all synaptic weights are positive.
\end{lemma}

\begin{proof}
By Theorem \ref{thm:CPWL_mapping}, we observe that $t_\Phi$ is a PWL mapping. It can be inferred via \eqref{eqn:firing_time_specific} that each linear region corresponds to a subset of input neurons responsible for the firing of $v$ on that specific domain. Hence, the number of regions is bounded by the number of non-empty subsets of $\{u_1,\dots,u_d\}$, i.e., $2^d-1$. Now, observe that any subset of input neurons causes a spike in $v$ if and only if the sum of their weights is positive. Otherwise, inputs from the corresponding region cannot trigger a spike in $v$ since their net contribution is negative, i.e., the potential does not reach the threshold $\theta_v$. Hence, the maximal number of regions is attained if and only if all weights are positive, and thereby the sum of weights of any subset of input neurons is positive as well.
\end{proof}

A one-layer ReLU-ANN with one output neuron will partition the input domain into at most two linear regions, independent of the dimension of the input. In contrast, for a one-layer LSNN with one output neuron, the maximum number of linear regions scales exponentially in the input dimension.
In the case of LSNNs, non-linearity is an intrinsic property of the model and emerges from the subset of neurons that affect the firing time of the output neuron, whereas in ANNs a non-linear activation is directly applied to the output neuron. By shifting the non-linearity and applying it to the input, ANNs could exhibit the same exponential scaling of the linear regions as LSNNs. However, this change has rather a detrimental effect on the expressivity since the partitioning of the input domain is fixed and independent of the parameters of the ANN. The flexibility of LSNNs to generate arbitrary linear regions is to a certain extent limited, albeit not entirely restricted as in the adjusted ANN; this is exemplarily demonstrated for a two-dimensional input space in Appendix \ref{sec:realizationSNNs_appendix}.

Analogously to the result in Theorem \ref{thm:approx_SNN_from_ANN}, a natural question is: what is the complexity of the ReLU ANN needed to emulate a given LSNN? The full power of ANN comes into play with large numbers of layers, however, the result in Lemma \ref{prop:Partition} suggests that a shallow LSNN can be as expressive as a deep ReLU network in terms of the number of linear regions.
In the following, we give upper bounds on depth and number of computational units required for a ReLU-ANN to express a one-layer LSNN with $d-$dimensional input.

\begin{theorem}\label{prop:lower_bound}
For $d \geq 2$, $\ell := \lceil\log_2(d+1)\rceil + 1$. Let $\Phi$ be a one-layer LSNN with one output neuron $v$ and $d$ input neurons $u_1, \dots, u_d$ with $w_{u_iv} \in \R_{>0}$ for $i \in [d]$. Then, 
\begin{enumerate}
\item[(a)] $t_\Phi$ can be realized by a ReLU-ANN $\Psi$ with $L(\Psi) = \ell$ and $N(\Psi) \in \mathcal{O}(\ell\cdot 2^{2d^3 + 3d^2 + d})$. 
\item[(b)] $t_\Phi$ can be realized by a ReLU-ANN $\Psi$ with $L(\Psi) \in \mathcal{O}(d)$ and $N(\Psi) \in \mathcal{O}(8^d)$.
\end{enumerate}
\end{theorem}

\begin{proof}
    Since $\Phi$ consists of only positive weights, the firing map $t_\Phi$ is via Theorem \ref{thm:CPWL_mapping} and Lemma \ref{prop:Partition} a CPWL mapping with $2^d-1$ linear regions. Using upper bounds on the size of ReLU-ANNs to realize CPWL mappings with a fixed number of linear regions, we obtain the given bounds. Thereby, the result (a) follows from Theorem 9 in \cite{hertrich2023lower}, and (b) follows from Theorem 1 in \cite{chen2023improved}. 
\end{proof}

The complexity bounds in (a) and (b) of Theorem \ref{prop:lower_bound} are connected to the number of linear regions of the CPWL function that the given LSNN realizes. In (a), to represent this CPWL function with lower depth, a substantial number of units in each layer might be necessary. Conversely, in (b), deep but less wide networks can achieve the same function. To the best of our knowledge, we are not aware of any lower bounds on the depth of the ReLU-ANN in terms of the number of linear regions when expressing any CPWL function. Extending the bounds to multi-layer LSNNs is an important step and is left for future investigation.

Via Theorem \ref{thm:approx_SNN_from_ANN} and Theorem \ref{prop:lower_bound}, we provide upper bounds on the size of LSNNs to realize ReLU networks and vice versa. Although the bounds presented in Theorem \ref{thm:approx_SNN_from_ANN} and Theorem \ref{prop:lower_bound} are not optimal, however, note that the bound on the size of ReLU-ANNs to emulate one-layer LSNNs grows exponentially in the input dimension whereas for LSNNs to emulate $L$-layer ReLU-ANNs, the bound scales linearly in the input dimension.
These bounds (and their derivation) suggest that LSNNs and ReLU-ANNs offer distinct benefits for realizing certain types of CPWL functions. 
These observations highlight the need to establish the associated lower bounds. This will not only shed light on the types of functions for which LSNNs are better suited in terms of emulation / approximation capabilities than ReLU-ANNs but also provide valuable insights into their computational power.

\section{Discussion}
\label{section:discussion}
The central aim of this paper is to study and compare the expressive power of SNNs and ANNs employing any PWL activation function. 
Our expressivity result in Theorem \ref{thm:approx_SNN_from_ANN} implies that LSNNs can approximate any function with the same accuracy and a certain complexity overhead as (deep) ANNs employing a piecewise linear activation function, given the response function satisfies some basic assumptions.
Most related to Theorem \ref{thm:approx_SNN_from_ANN} are the results in \citet{relutosnn2022ana}. 
Under certain assumptions, the authors define a one-to-one neuron mapping that converts a trained ReLU network to a corresponding SNN consisting of integrate and fire neurons by a non-linear transformation of parameters. However, significant distinctions exist between the approaches, particularly in terms of the chosen model, for instance, with the handling of the threshold parameter. 
In terms of methodology, we introduce an auxiliary neuron to ensure the firing of neurons even when a corresponding ReLU neuron exhibits zero activity. This diverges from their approach, which employs external current and a special parameter to achieve similar outcomes. We study the differences in the structure of computations between ANNs and SNNs, whereas in \citet{relutosnn2022ana}, only the conversion of ANNs to SNNs is examined and not vice versa.

Rather than approximating some function space by emulating a known construction for ReLU networks, one could also achieve optimal approximations by leveraging the intrinsic capabilities of LSNNs instead. The findings in Lemma \ref{prop:Partition} and Theorem \ref{prop:lower_bound} indicate that the latter approach may indeed be beneficial in terms of the complexity of the architecture in certain circumstances. However, we point out that finding optimal architectures for approximating different classes of functions is not the focal point of our work. 
The significance of our results lies in investigating theoretically the approximation and expressivity capabilities of SNNs, highlighting their potential as an alternative computational model for complex tasks. 
Extending the model of an LSNN neuron by incorporating, e.g., multiple spikes of a neuron, may yield an improvement in our results. However, by increasing the complexity of the model the analysis also tends to be more elaborate. In the aforementioned case of multiple spikes the threshold function becomes important so that additional complexity when approximating some target function is introduced since one would have to consider the coupled effect of response and threshold functions. Similarly, the choice of the response function and the frequency of neuron firings will surely influence the approximation results and we leave this for future work.

\paragraph{Limitations} 
We study similarities and differences in the structure of computations between ANNs and LSNNs and theoretically show that LSNNs can be as expressive as ReLU-ANNs.
However, achieving similar results in practice heavily relies on the effectiveness of the employed training algorithms.
The implementation of efficient learning algorithms with weights, delays, and thresholds as programmable parameters is left for future work.
In this work, our choice of model resides on theoretical considerations and not on practical considerations regarding implementation. However, there might be other models of spiking neurons that are more apt for implementation purposes --- see e.g. \citet{relutosnn2022ana} and \citet{temporal_single_spike_backprop2020_comsa}.
Furthermore, in reality, due to the ubiquitous sources of noise in the spiking neurons, the firing activity of a neuron is not deterministic. 
For mathematical simplicity, we perform our analysis in a noise-free case.
Generalizing to the case of noisy spiking neurons is important (for instance concerning the aforementioned implementation in noisy environments) and may lead to further insights into the model.

\acks{
The authors would like to thank Philipp Petersen for helpful discussion and feedback. M. Singh is supported by the DAAD programme Konrad Zuse Schools of Excellence in Artificial Intelligence, sponsored by the Federal Ministry of Education and Research. 

G. Kutyniok acknowledges support from LMUexcellent, funded by the Federal Ministry of Education and Research (BMBF) and the Free State of Bavaria under the Excellence Strategy of the Federal Government and the Länder as well as by the Hightech Agenda Bavaria. Further, G. Kutyniok was supported in part by the DAAD programme Konrad Zuse Schools of Excellence in Artificial Intelligence, sponsored by the Federal Ministry of Education and Research. G. Kutyniok also acknowledges support from the Munich Center for Machine Learning (MCML) as well as the German Research Foundation under Grants DFG-SPP-2298, KU 1446/31-1 and KU 1446/32-1 and under Grant DFG-SFB/TR 109, Project C09 and the German Federal Ministry of Education and Research (BMBF) under Grant MaGriDo.}

\bibliography{bib_file}

\newpage

\appendix


\section{Proofs}
\label{section:Appendix}

\paragraph{Outline} 
We start by introducing the spiking network calculus in Section \ref{section:SNN_calculus} to compose and parallelize different networks. 
In Section \ref{sec:realizationSNNs_appendix}, we characterize the firing maps of LSNNs. In particular, we 
show that the firing maps of LSNNs are PWL and under stronger assumptions continuous, increasing, and concave. In Section \ref{sec:RealizingReLU_app}, we construct an LSNN that emulates the ReLU non-linearity, and subsequently in Section \ref{sec:approx_relu_network_using_SNN}, we prove that an LSNN can realize the output of any ReLU network and simultaneously provide bounds on the required size of the LSNN.

\subsection{Spiking neural network calculus} 
\label{section:SNN_calculus}
It can be observed from Definition \ref{definition:encoding} that both inputs and outputs of LSNNs are encoded in a unified format.
This characteristic is crucial for concatenating/parallelizing two spiking network architectures that further enable us to attain compositions of network realizations.  

We operate in the following setting: Let $L_1$, $L_2, d_1, d_2, d_1^\prime, d_2^\prime \in \N$. Consider two LSNNs $\Phi_1$, $\Phi_2$ given by 
\begin{equation*}
    \Phi_i = ((W_1^i, D_1^i, \Theta_1^i), \dots, (W_{L_i}^i, D_{L_i}^i, \Theta_{L_i}^i)), \quad i= 1,2,    
\end{equation*}
with input domains $[a_1,b_1]^{d_1} \subset \R^{d_1}$, $[a_2,b_2]^{d_2} \subset \R^{d_2}$ and output dimension $d_1^\prime, d_2^\prime$, respectively. Denote the input neurons by $u_1,\dots, u_{d_i}$ with respective firing times $t^i_{u_j}$. 
By Definition \ref{def:firingMap} and \ref{definition:encoding}, we can express the firing times of the input neurons as 
\begin{align}\label{eqn:firingtimephi1phi2}
    t_u^1(x) := (t^1_{u_1}, \dots, t^1_{u_{d_1}})^T &= T^1_{\text{in}} + x \quad \text{ for } x \in [a_1,b_1]^{d_1},   \nonumber\\
    t_u^2(x) := (t^2_{u_1}, \dots, t^2_{u_{d_2}})^T &= T^2_{\text{in}} + x \quad \text{ for } x \in [a_2,b_2]^{d_2}
\end{align}  
and the realization of the networks as
\begin{align}\label{eq:realizationphi1phi2}
    \cR_{\Phi_1}(x) &= - T^1_{\text{out}} + t_{\Phi_1}(t^1_u(x)) \quad \text{ for } x \in [a_1,b_1]^{d_1}, \nonumber\\
    \cR_{\Phi_2}(x) &= - T^2_{\text{out}} + t_{\Phi_2}(t^2_u(x)) \quad \text{ for } x \in [a_2,b_2]^{d_2}
\end{align}
for some constants $T^1_{\text{in}} \in \R^{d_1}$, $T^2_{\text{in}} \in \R^{d_2}$, $T^1_{\text{out}} \in \R^{d^\prime_1}$, $T^2_{\text{out}} \in \R^{d^\prime_2}$.

We define the concatenation of the two networks in the following way. 

\begin{definition} (Concatenation)
    \label{defn:concatenation}
    Let $\Phi_1$ and $\Phi_2$ be such that the input layer of $\Phi_1$ has the same dimension as the output layer of $\Phi_2$, i.e., $d_2^\prime = d_1$. Then, the concatenation of $\Phi_1$ and $\Phi_2$, denoted as $\Phi_1 \bullet \Phi_2$, represents the $(L_1 + L_2)$-layer network
    \begin{equation*}
    \Phi_1\bullet \Phi_2 := ((W_1^2, D_1^2, \Theta_1^2), \dots, (W_{L_2}^2, D_{L_2}^2, \Theta_{L_2}^2), (W_1^1, D_1^1, \Theta_1^1), \dots, (W_{L_1}^1, D_{L_1}^1, \Theta_{L_1}^1)).
\end{equation*}
\end{definition}

\begin{lemma}\label{lemma:concatenation} 
Let $d_2^\prime = d_1$ and fix $T_{\text{in}} = T^2_{\text{in}}$ and $T_{\text{out}} = T^1_{\text{out}}$. If $T^2_{\text{out}} = T^1_{\text{in}}$ and $\cR_{\Phi_2}([a_2,b_2]^{d_2}) \subset [a_1,b_1]^{d_1}$, then 
\begin{equation*}
    \cR_{\Phi_1\bullet \Phi_2}(x) = \cR_{\Phi_1}(\cR_{\Phi_2}(x)) \quad \text{for all} \:\: x \in [a,b]^{d_2}
\end{equation*}
with respect to the reference times $T_{\text{in}}, T_{\text{out}}$.
Moreover, $\Phi_1\bullet \Phi_2$ is composed of $N(\Phi_1) + N(\Phi_2) - d_1$ computational units.
\end{lemma}

\begin{proof}
It is straightforward to verify via the construction that the network $\Phi_1 \bullet \Phi_2$ is composed of $N(\Phi_1) + N(\Phi_2) - d_1$ computational units. Moreover, under the given assumptions $\cR_{\Phi_1} \circ \cR_{\Phi_2}$ is well-defined so that \eqref{eqn:firingtimephi1phi2} and \eqref{eq:realizationphi1phi2} imply
\begin{align*}
    \cR_{\Phi_1 \bullet \Phi_2}(x) &= -T_{\text{out}} +  t_{\Phi_1}(t_{\Phi_2}(T_{\text{in}} + x)) = -T^1_{\text{out}} + t_{\Phi_1}(t_{\Phi_2}(T^2_{\text{in}} + x)) \\
    &= -T^1_{\text{out}}  + t_{\Phi_1}(t_{\Phi_2}(t^2_u(x)))  = -T^1_{\text{out}}  + t_{\Phi_1}(T^2_{\text{out}} + \cR_{\Phi_2}(x)) \\
    &= - T^1_{\text{out}} + t_{\Phi_1}(T^1_{\text{in}} + \cR_{\Phi_2}(x)) =  - T^1_{\text{out}} + t_{\Phi_1}(t^1_u(\cR_{\Phi_2}(x)))\\
    &=  \cR_{\Phi_1}(\cR_{\Phi_2}(x)) \quad \text{ for } x \in [a_2,b_2]^{d_2}.
\end{align*}
\end{proof}

In addition to concatenating networks, we also perform parallelization operation on LSNNs.
\begin{definition}(Parallelization)\label{defn:parallelization}
Let $\Phi_1$ and $\Phi_2$ be such that they have the same depth and input dimension, i.e., $L_1 = L_2 =: L$ and $d_1 = d_2 =: d$. Then, the parallelization of $\Phi_1$ and $\Phi_2$, denoted as $P(\Phi_1, \Phi_2)$, represents the $L$-layer network with $d$-dimensional input
\begin{equation*}
    P(\Phi_1, \Phi_2) := ((\Tilde{W}_1, \Tilde{D}_1, \Tilde{\Theta}_1), \dots, (\Tilde{W}_L, \Tilde{D}_L, \Tilde{\Theta}_L)), 
\end{equation*}
where 
\begin{equation*}
\Tilde{W}_1 = \begin{pmatrix}
W_1^1 & W_1^2\\[3pt]
\end{pmatrix}, \quad 
\Tilde{D}_1 = \begin{pmatrix}
D_1^1 & D_1^2\\[3pt]
\end{pmatrix}, \quad    
\Tilde{\Theta}_1 = \begin{pmatrix}
\Theta_1^1\\[3pt]
\Theta_1^2\\
\end{pmatrix}
\end{equation*}
and 
\begin{equation*}
\Tilde{W}_l = \begin{pmatrix}
W_l^1 & 0\\
0 & W_l^2\\
\end{pmatrix}, \quad
\Tilde{D}_l = \begin{pmatrix}
D_l^1 & 0\\
0 & D_l^2\\
\end{pmatrix}, \quad
\Tilde{\Theta}_l = \begin{pmatrix}
\Theta_l^1\\[2pt]
\Theta_l^2\\
\end{pmatrix}, \quad \text{for} \:\: 1 < l \leq L.
\end{equation*}
\end{definition}

\begin{lemma}
\label{lemma:parallelization}
Let $d:= d_2 = d_1$ and fix $T_{\text{in}} := T^1_{\text{in}}$, $T_{\text{out}} := (T^1_{\text{out}},T^2_{\text{out}})$, $a:=a_1$ and $b:=b_1$. If $T^2_{\text{in}} = T^1_{\text{in}}$, $T^2_{\text{out}} = T^1_{\text{out}}$ and $a_1=a_2$, $b_1=b_2$, then
\begin{equation*}
    \cR_{P(\Phi_1, \Phi_2)}(x) = (\cR_{\Phi_1}(x), \cR_{\Phi_2}(x)) \quad \text{ for }  x \in [a,b]^d 
\end{equation*}
with respect to the reference times $T_{\text{in}}, T_{\text{out}}$.
Moreover, $P(\Phi_1, \Phi_2)$ is composed of $N(\Phi_1) + N(\Phi_2) - d$ computational units.
\end{lemma}

\begin{proof}
    The number of computational units is an immediate consequence of the construction. Since the input domains of $\Phi_1$ and $\Phi_2$ agree, \eqref{eqn:firingtimephi1phi2} and \eqref{eq:realizationphi1phi2} show that
\begin{align*}
    \cR_{P(\Phi_1, \Phi_2)}(x) &= -T_{\text{out}} + (t_{\Phi_1}(T_{\text{in}} + x), t_{\Phi_2}(T_{\text{in}} + x))\\
    &= (-T^1_{\text{out}} + t_{\Phi_1}(T^1_{\text{in}} + x), -T^2_{\text{out}} + t_{\Phi_2}(T^2_{\text{in}} + x)) \\
    &=(-T^1_{\text{out}} + t_{\Phi_1}(t^1_u(x)), -T^2_{\text{out}} + t_{\Phi_2}(t_u^2(x)))\\
    &= (\cR_{\Phi_1}(x), \cR_{\Phi_2}(x))  \quad \text{ for } x \in [a,b]^{d}.
\end{align*}

\end{proof}

\begin{remark}
    Note that parallelization and concatenation can be straightforwardly extended (recursively) to a finite number of networks. Additionally, more general forms of parallelization and concatenations of networks, e.g., parallelization of networks with different depths, can be established. However, for the constructions presented in this work, the introduced notions suffice.
\end{remark}

\subsection{Characterization of functions expressed by LSNNs}
\label{sec:realizationSNNs_appendix}

\subsubsection{Spiking neuron with two inputs}
First, we provide a simple toy example to demonstrate the dynamics of an LSNN neuron. Let $v$ be an LSNN neuron with two input neurons $u_1,u_2$. Denote the associated weights and delays by $w_{u_i v} \in \R$ and $d_{u_i v} \geq 0$, respectively, and the threshold of $v$ by $\theta_v >0$. A spike emitted from $v$ could then be caused by either $u_1$ or $u_2$ or a combination of both. Each possibility corresponds to a linear region in the input space $\R^2$. We consider each case separately under Assumption I, i.e., $\delta$ in \eqref{eq:Response_cond} is arbitrarily large, and we discuss the implications of this assumption in more detail after presenting the different cases. 

\textbf{Case 1}: $u_1$ causes $v$ to spike before a potential effect from $u_2$ reaches $v$. Note that this can only happen if $w_{u_1 v} > 0$ and
\begin{equation*}
    t_{u_2} + d_{u_2 v} \geq t_v = \frac{\theta_v}{w_{u_1 v}} +t_{u_1} + d_{u_1 v},
\end{equation*}
where we applied \eqref{eqn:firing_time} and \eqref{eqn:firing_time_specific}, and $t_z$ represents the firing time of a neuron $z$. Solving for $t_{u_2}$ leads to
\begin{equation*}
    t_{u_2} \geq \frac{\theta_v}{w_{u_1 v}} +t_{u_1} + d_{u_1 v} - d_{u_2 v}.
\end{equation*}

\textbf{Case 2}: An analogous calculation shows that 
\begin{equation*}
    t_{u_2} \leq  - \frac{\theta_v}{w_{u_2 v}} +t_{u_1} + d_{u_1 v} - d_{u_2 v},
\end{equation*}
whenever $u_2$ causes $v$ to spike before a potential effect from $u_1$ reaches $v$.

\textbf{Case 3}: The remaining possibility is that spikes from $u_1$ and $u_2$ influence the firing time of $v$. Then, the following needs to hold: $w_{u_1 v} + w_{u_2 v} > 0$ and 
\begin{align*}
    t_{u_1} + d_{u_1 v} &< t_v = \frac{\theta_v}{w_{u_1 v}+ w_{u_2 v}} + \sum_i \frac{w_{u_i v}}{w_{u_1 v}+ w_{u_2 v}} (t_{u_i} + d_{u_i v}) \quad \text{ and }\\
    t_{u_2} + d_{u_2 v} &< t_v = \frac{\theta_v}{w_{u_1 v}+ w_{u_2 v}} +\sum_i \frac{w_{u_i v}}{w_{u_1 v}+ w_{u_2 v}} (t_{u_i} + d_{u_i v}).    
\end{align*}
This yields
\begin{align*}
    t_{u_2} &\begin{cases}
                > - \frac{\theta_v}{w_{u_2 v}} +t_{u_1} + d_{u_1 v} - d_{u_2 v},& \text{ if } \frac{w_{u_2 v}}{w_{u_1 v}+ w_{u_2 v}} > 0\\[7pt]
                < - \frac{\theta_v}{w_{u_2 v}} +t_{u_1} + d_{u_1 v} - d_{u_2 v},& \text{ if } \frac{w_{u_2 v}}{w_{u_1 v}+ w_{u_2 v}} < 0 
            \end{cases},
            \\[5pt]
            \text{respectively}&\\[5pt]
    t_{u_2} &\begin{cases}
                < \frac{\theta_v}{w_{u_1 v}} +t_{u_1} + d_{u_1 v} - d_{u_2 v},& \text{ if } \frac{w_{u_1 v}}{w_{u_1 v}+ w_{u_2 v}} > 0\\[7pt]
                > \frac{\theta_v}{w_{u_1 v}} +t_{u_1} + d_{u_1 v} - d_{u_2 v},& \text{ if } \frac{w_{u_1 v}}{w_{u_1 v}+ w_{u_2 v}} < 0
            \end{cases}.       
\end{align*}

\begin{example}
\label{example:linear_regions}
    In a simple setting with $\theta_v=w_{u_i v}= d_{u_2 v}  =1$ and $d_{u_1 v}=2 $, the above considerations imply the following firing time of $v$ on the corresponding linear regions (see Figure \ref{fig:linRegEx}):
    \begin{equation*}
        t_v = \begin{cases}
                    t_{u_1} + 3, &\text{ if } t_{u_2} \geq t_{u_1} + 2 \\[4pt]
                    t_{u_2} + 2, &\text{ if } t_{u_2} \leq t_{u_1} \\[4pt]
                    \frac{1}{2}(t_{u_1}+t_{u_2}) + 2, &\text{ if } t_{u_1} < t_{u_2} < t_{u_1} + 2
              \end{cases}.
    \end{equation*}
\end{example}

\begin{figure}[ht] 
    \centering
    \begin{subfigure}[b]{0.44\linewidth}
        \includegraphics[width=\linewidth]{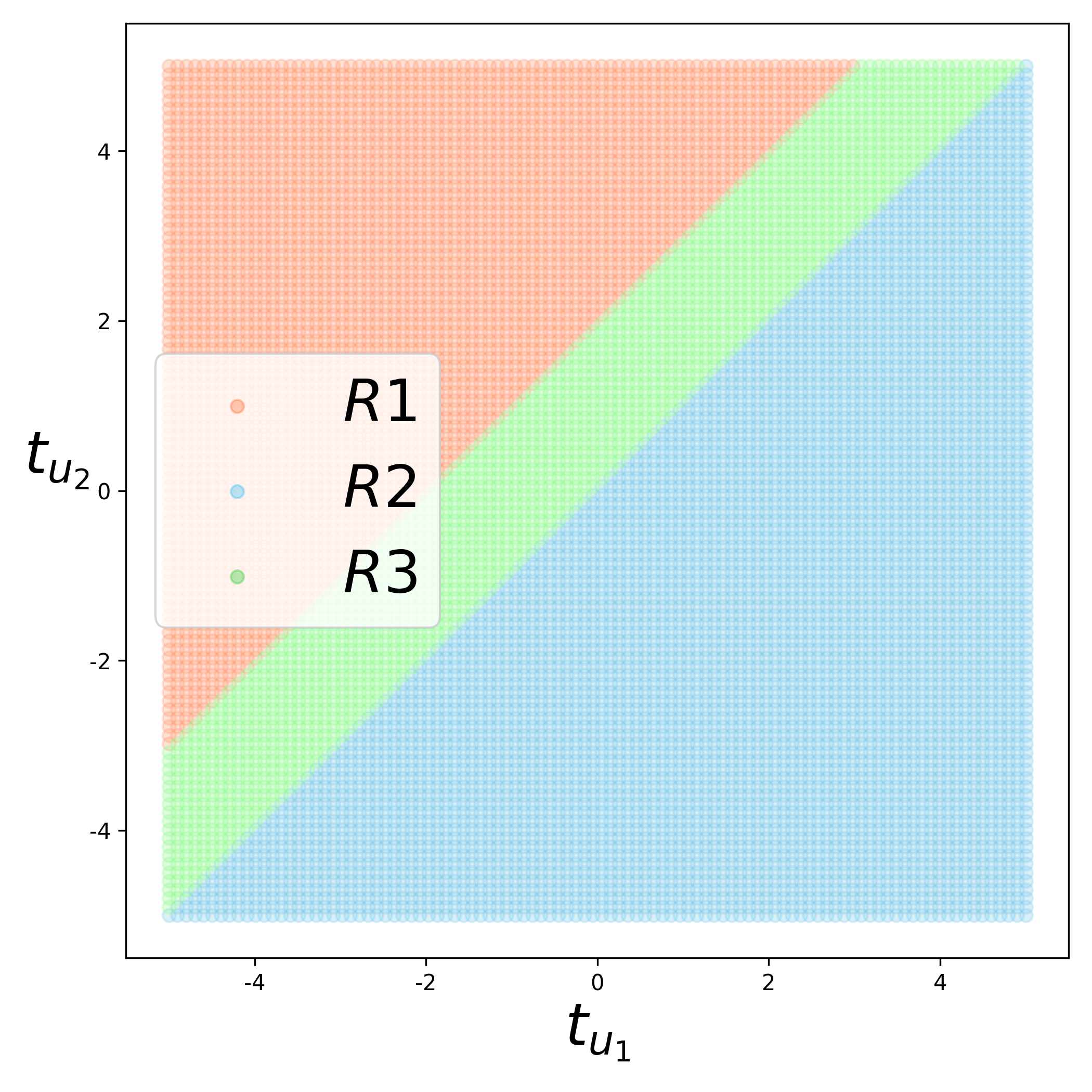}
        \caption{}
        \label{subfig:CPWL_regions}
    \end{subfigure}
    \begin{subfigure}[b]{0.53\linewidth}
        \includegraphics[width=\linewidth, height = 6.8cm]{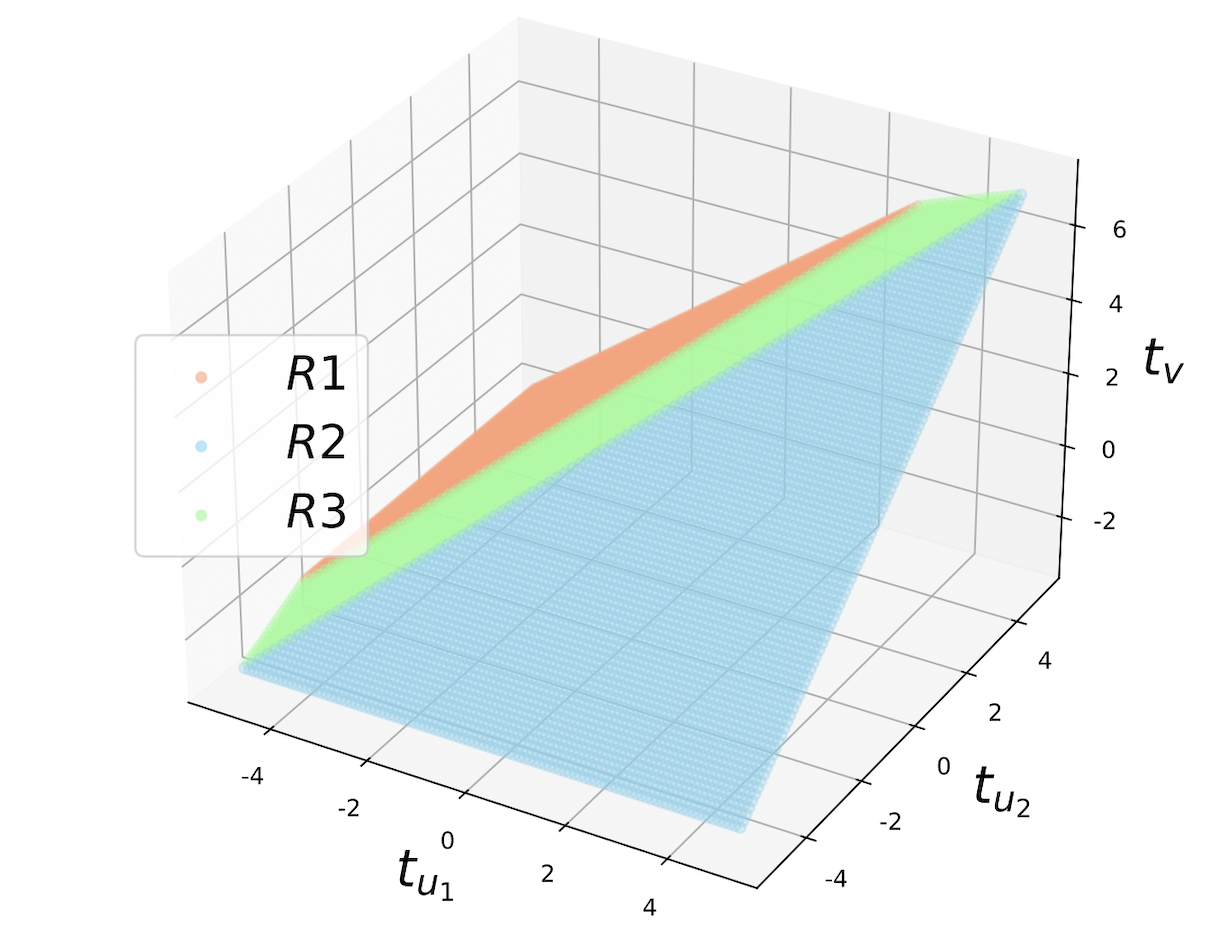}
        \caption{}
        \label{subfig:CPWL_regions_realization}
    \end{subfigure}
	\caption{Illustration of Example \ref{example:linear_regions}. It shows that the output firing time $t_v(t_{u_1}, t_{u_2})$ as a function of inputs $t_{u_1}, t_{u_2}$ is a CPWL mapping. (a) An illustration of the partitioning of the input space into three different regions. (b) Each region is associated with an affine-linear mapping.}
    \label{fig:linRegEx}
\end{figure}  

Already this simple setting with two-dimensional inputs provides crucial insights. The actual number of linear regions in the input domain corresponds to the parameter of the LSNN neuron $v$. In particular, the maximum number of linear regions, i.e. three, can only occur if both weights $w_{u_i v}$ are positive. Similarly, $v$ does not fire at all if both weights are non-positive, which motivates the condition in Assumption II. The exact number of linear regions depends on the sign and magnitude of the weights. Furthermore, note that the linear regions are described by hyperplanes of the form
\begin{equation}\label{eq:App_Boundary2}
    t_{u_2} \lesseqgtr t_{u_1} + C_{p,u},
\end{equation}
where $C_{p,u}$ is a constant depending on the parameter $p$ corresponding to $v$, i.e., threshold, delays and weights, and the actual input neuron(s) causing $v$ to spike. Hence, $p$ has only a limited effect on the boundary of a linear region; depending on their exact value, the parameter only introduces an additive constant shift. 
\begin{remark}\label{rmk:delta}
    Dropping the assumption that $\delta$ is arbitrarily large in \eqref{eq:Response_cond} yields an evolved model that is also biologically more realistic. The magnitude of $\delta$ describes the duration in which an incoming spike influences the membrane potential of a neuron. By setting $\delta$ arbitrarily large, we generally consider an incoming spike to have a lasting effect on the membrane potential. Specifying a fixed $\delta$ increases the importance of the timing of the individual spikes as well as the choice of the parameter. For instance, inputs from certain regions in the input domain may not trigger a spike since the combined effect of multiple delayed incoming spikes is neglected. An in-depth analysis of the influence of $\delta$ is left as future work and we continue our analysis under the assumption that $\delta$ is arbitrarily large.
\end{remark} 

\subsubsection{Spiking neuron with arbitrarily many inputs}

A significant observation in the two-dimensional case is that the firing time $t_v(t_{u_1},t_{u_2})$ as a function of the input $t_{u_1}, t_{u_2}$ is a CPWL mapping. Indeed, each linear region is associated with an affine linear mapping and crucially these affine mappings agree at the breakpoints. This intuitively makes sense since a breakpoint marks the event when the effect of an additional neuron on the firing time of $v$ needs to be taken into consideration or, equivalently, a neuron does not contribute to the firing of $v$ anymore. However, in both circumstances, the effective contribution of this specific neuron is zero (and the contribution of the other neuron remains unchanged) at the breakpoint so that the crossing of a breakpoint and the associated change of a linear region does not result in a discontinuity. 

Formally, the class of CPWL functions describes functions that are globally continuous and locally linear on each polytope in a given finite decomposition of $\R^d$ into polytopes. We refer to the polytopes as linear regions. The insights obtained in the two-dimensional case do not straightforwardly extend to a $d$-dimensional input domain for $d>2$. Crucially, continuity may be lost as the following simple example shows.
\begin{example}\label{example:DPWL}
    Let v be an LSNN neuron with threshold $\theta_v = 1$ and presynaptic neurons $u_1, u_2, u_3$ with corresponding weights $w_{u_1 v} = 1, w_{u_2 v} = -1, w_{u_3 v} = 1$, delays $d_{u_1 v} = d_{u_2 v} =  d_{u_3 v} = 0$, and firing times $t_{u_1}=0, t_{u_2}= 1 + \varepsilon, t_{u_3}=2$, respectively. One easily verifies via \eqref{eqn:firing_time} that 
    \begin{equation*}
        t_v(t_{u_1},t_{u_2},t_{u_3}) = 
            \begin{cases} 
                1,  &\text{ for } \varepsilon \geq 0\\
                2 - \varepsilon,  &\text{ for } \varepsilon < 0           
            \end{cases} .
    \end{equation*}
    Hence, $t_v(t_{u_1},t_{u_2},t_{u_3})$ is not continuous since
    \begin{equation*}
        \lim_{\varepsilon \uparrow 0} t_v(t_{u_1},t_{u_2},t_{u_3}) = \lim_{\varepsilon \uparrow 0} 2 - \varepsilon = 2 \neq 1 = \lim_{\varepsilon \downarrow 0} t_v(t_{u_1},t_{u_2},t_{u_3}).
    \end{equation*}
\end{example}
However, we can show that an LSNN neuron generates a PWL mapping, which under certain conditions on its weights is in fact continuous.
\begin{lemma}\label{lemma:pwlNeuron}
    Let $v$ be a LSNN neuron with with threshold $\theta_v >0$ and presynaptic neurons $u_1, \dots, u_d$ with corresponding weights $w_{u_i v} \in \R$, delays $d_{u_i v}\geq 0$, and firing times $t_{u_i} \in \R$, respectively. Then the firing time $t_v(t_{u_1},\dots,t_{u_d})$ as a function of the firing times $t_{u_1},\dots,t_{u_d}$ is a PWL mapping, and additionally continuous provided that 
    \begin{equation}\label{eq:contCondv}
        w_{u_i v} + \sum_{j : w_{u_j v} \leq 0} w_{u_j v} >0 \quad \text{ for all } i \text{ with } w_{u_i v} > 0.
\end{equation}
\end{lemma}
\begin{proof}
    Recall that we operate under Assumption II, i.e., we presuppose that $\sum_{i=1}^d w_{u_i v} >0$ so that any input firing time $(t_{u_1},\dots,t_{u_d}) \in \R^d$ necessarily triggers a firing in $v$. In particular, the notion of $t_v$ as a PWL mapping on $\R^d$ is well-defined. Moreover, for given $t_{u_1},\dots,t_{u_d}$ we can identify a subset $I\subset \{1,\dots,d\}$ such that all $u_i$ with $i\in I$ contribute to the firing of $v$ whereas spikes from $u_j$ with $j \in I^c= \{1,\dots,d\}\setminus I$ do not influence the firing of $v$ (since these spikes arrive after $v$ already fired). Then $\sum_{i\in I} w_{u_i v}$ is required to be positive, so that by \eqref{eqn:firing_time} and \eqref{eqn:firing_time_specific} the following holds:
    \begin{equation}\label{eq:App_FT}
        t_{u_k} + d_{u_k v} \geq t_v = \frac{\theta_v}{\sum_{i\in I} w_{u_i v}} + \sum_{i\in I} \frac{w_{u_i v}}{\sum_{j\in I} w_{u_j v}} (t_{u_i} + d_{u_i v}) \quad \text{ for all } k \in I^c
    \end{equation}
    and 
    \begin{equation}\label{eq_App_FT2}
        t_{u_k} + d_{u_k v} < t_v = \frac{\theta_v}{\sum_{i\in I} w_{u_i v}} + \sum_{i\in I} \frac{w_{u_i v}}{\sum_{j\in I} w_{u_j v}} (t_{u_i} + d_{u_i v}) \quad \text{ for all } k \in I.    
    \end{equation}
    Rewriting yields
    \begin{equation}\label{eq:App_LinReg}
        t_{u_k} \geq \frac{\theta_v}{\sum_{i\in I} w_{u_i v}} + \sum_{i\in I} \frac{w_{u_i v}}{\sum_{j\in I} w_{u_j v}} (t_{u_i} + d_{u_i v}) - d_{u_k v}  \quad \text{ for all } k \in I^c
    \end{equation}
    and 
    \begin{equation*}
        t_{u_k}  
        \begin{cases}
                    < \frac{\theta_v}{\sum_{j\in I\setminus k} w_{u_j v}} + \sum_{i\in I\setminus k} \frac{w_{u_i v}}{\sum_{j\in I\setminus k} w_{u_j v}} (t_{u_i} + d_{u_i v}) - d_{u_k v},& \text{ if } \frac{\sum_{i\in I\setminus k} w_{u_i v}}{\sum_{i\in I} w_{u_i v}} > 0\\[8pt]
                    > \frac{\theta_v}{\sum_{j\in I\setminus k} w_{u_j v}} + \sum_{i\in I\setminus k} \frac{w_{u_i v}}{\sum_{j\in I\setminus k} w_{u_j v}} (t_{u_i} + d_{u_i v}) - d_{u_k v},& \text{ if } \frac{\sum_{i\in I\setminus k} w_{u_i v}}{\sum_{i\in I} w_{u_i v}} < 0 
                \end{cases} \forall k \in I.
    \end{equation*}
    It is now clear that the firing time $t_v(t_{u_1},\dots,t_{u_d})$ as a function of the input $t_{u_1},\dots, t_{u_d}$ is a piecewise linear mapping on polytopes decomposing $\R^d$. To show that the mapping is additionally continuous if \eqref{eq:contCondv} holds, we need to assess $t_v$ on the breakpoints. Let $I, J \subset \{1,\dots,d\}$ be index sets corresponding to input neurons $\{u_i : i\in I\}$,$\{u_j : j\in J\}$ that cause $v$ to fire on the input region $R^I \subset \R^d$, $R^J \subset \R^d$ respectively. Assume that it is possible to transition from $R^I$ to $R^J$ through a breakpoint $t^{I,J} = (t_{u_1}^{I,J}, \dots, t_{u_d}^{I,J}) \in \R^d$ without leaving $R^I \cup R^J$. Crossing the breakpoint is equivalent to the fact that the input neurons $\{u_i : i \in I\setminus J\}$ do not contribute to the firing of $v$ anymore and the input neurons $\{u_i : i \in J\setminus I\}$ begin to contribute to the firing of $v$. Subsequently, we consider different cases concerning the relation of $I$ and $J$. Thereby, we first require that all weights are positive.
   
    Now, assume that $J \subset I$. Then, we observe that the breakpoint $t^{I,J}$ is necessarily an element of the linear region corresponding to the index set with smaller cardinality, i.e., $t^{I,J} \in R^J$. This is an immediate consequence of \eqref{eq_App_FT2} and the fact that $t^{I,J}$ is characterized by
    \begin{equation}\label{eq:App_bp}
        t^{I,J}_{u_k} + d_{u_k v} = t_v(t^{I,J})  \quad \text{ for all } k \in I\setminus J.
    \end{equation}
    Indeed, if $t^{I,J}_{u_k} + d_{u_k v} > t_v(t^{I,J})$, then there exists $\varepsilon_k>0$ such that \eqref{eq:App_LinReg} also holds for $t^{I,J}_{u_k} \pm \varepsilon$, where $0 \leq \varepsilon < \varepsilon_k$, i.e., a small change in $t^{I,J}_{u_k}$ is not sufficient to change the corresponding linear region, contradicting our assumption that $t^{I,J}$ is a breakpoint. 

    The firing time $t_v(t^{I,J})$ is explicitly given by
    \begin{equation*}
        t_v(t^{I,J}) = \frac{\theta_v}{\sum_{i\in J} w_{u_i v}} + \sum_{i\in J} \frac{w_{u_i v}}{\sum_{j\in J} w_{u_j v}} (t_{u_i}^{I,J} + d_{u_i v})
    \end{equation*}
    Using \eqref{eq:App_bp}, we obtain
    \begin{equation*}
        0 =  -\frac{w_{u_k v}}{\sum_{j\in J} w_{u_j v}}(t_v(t^{I,J}) - (t^{I,J}_{u_k} + d_{u_k v})) \quad \text{ for all } k \in I\setminus J
    \end{equation*}
    so that
    \begin{equation*}
        t_v(t^{I,J}) = \frac{\theta_v}{\sum_{i\in J} w_{u_i v}} + \sum_{i\in J} \frac{w_{u_i v}}{\sum_{j\in J} w_{u_j v}} (t^{I,J}_{u_i} + d_{u_i v})  -\sum_{i\in I\setminus J}\frac{w_{u_i v}}{\sum_{j\in J} w_{u_j v}}(t_v(t^{I,J}) - (t^{I,J}_{u_i} + d_{u_i v})) . 
    \end{equation*}
    Solving for $t_v(t^{I,J})$ yields
    \begin{align*}
        t_v(t^{I,J}) &= \Big(1 + \sum_{i\in I\setminus J}\frac{w_{u_i v}}{\sum_{j\in J} w_{u_j v}}\Big)^{-1}\cdot\Big(\frac{\theta_v}{\sum_{i\in J} w_{u_i v}} + \sum_{i\in I} \frac{w_{u_i v}}{\sum_{j\in J} w_{u_j v}} (t^{I,J}_{u_i} + d_{u_i v}) \Big)  \\ 
        &= \sum_{i\in J}\frac{w_{u_i v}}{\sum_{j\in I} w_{u_j v}} \cdot \Big(\frac{\theta_v}{\sum_{i\in J} w_{u_i v}} + \sum_{i\in I} \frac{w_{u_i v}}{\sum_{j\in J} w_{u_j v}} (t^{I,J}_{u_i} + d_{u_i v}) \Big)\\
        &= \frac{\theta_v}{\sum_{i\in I} w_{u_i v}} + \sum_{i\in I} \frac{w_{u_i v}}{\sum_{j\in I} w_{u_j v}} (t^{I,J}_{u_i} + d_{u_i v}),
    \end{align*}
    which is exactly the expression for the firing time on $R^I$. This shows that $t_v(t_{u_1},\dots,t_{u_d})$ is continuous in $t^{I,J}$. Since the breakpoint $t^{I,J}$ was chosen arbitrarily, $t_v(t_{u_1},\dots,t_{u_d})$ is continuous at any breakpoint.

    The case $I \subset J$ follows analogously. It remains to check the case when neither $I \subset J$ nor $J \subset I$. To that end, let $i^\ast \in I\setminus J$ and $j^\ast \in J\setminus I$. Assume without loss of generality that $t^{I,J} \in R^I$ so that \eqref{eq:App_FT} and \eqref{eq_App_FT2} imply
    \begin{equation*}
        t^{I,J}_{u_{i^\ast}} + d_{u_{i^\ast} v} < t_v(t^{I,J}) \leq t^{I,J}_{u_{j^\ast}} + d_{u_{j^\ast} v}.    
    \end{equation*}
    Hence, there exists $\varepsilon>0$ such that 
    \begin{equation}\label{eq:App_epsBall}
        t^{I,J}_{u_{i^\ast}} + d_{u_{i^\ast} v} < t^{I,J}_{u_{j^\ast}} + d_{u_{j^\ast} v} - \varepsilon.
    \end{equation}
    Moreover, due to the fact that $t^{I,J}$ is a breakpoint we can find $t^J \in R^J \cap \mathcal{B}(t^{I,J}; \frac{\varepsilon}{3})$, where $\mathcal{B}(t^{I,J}; \frac{\varepsilon}{3})$ denotes the open ball with radius $\frac{\varepsilon}{3}$ centered at $t^{I,J}$. In particular, this entails that
    \begin{equation*}
        -\frac{\varepsilon}{3} < (t^{J}_{u_{i^\ast}} - t^{I,J}_{u_{i^\ast}}), (t^{I,J}_{u_{j^\ast}} - t^{J}_{u_{j^\ast}}) < \frac{\varepsilon}{3},
    \end{equation*}
    and therefore together with \eqref{eq:App_epsBall}
    \begin{align*}
        t^{J}_{u_{i^\ast}} + d_{u_{i^\ast} v} - (t^{J}_{u_{j^\ast}} + d_{u_{j^\ast} v}) &=  (t^{J}_{u_{i^\ast}} - t^{I,J}_{u_{i^\ast}}) + (t^{I,J}_{u_{i^\ast}} + d_{u_{i^\ast} v}  - (t^{I,J}_{u_{j^\ast}} + d_{u_{j^\ast} v})) + (t^{I,J}_{u_{j^\ast}} - t^{J}_{u_{j^\ast}})   \\
        &< 0, \quad \text{i.e., } t^{J}_{u_{i^\ast}} + d_{u_{i^\ast} v} < t^{J}_{u_{j^\ast}} + d_{u_{j^\ast} v}.
    \end{align*}
    However, \eqref{eq:App_FT} and \eqref{eq_App_FT2} require that
    \begin{equation*}
        t^{J}_{u_{j^\ast}} + d_{u_{j^\ast} v} < t_v(t^J) \leq t^{J}_{u_{i^\ast}} + d_{u_{i^\ast} v}     
    \end{equation*}
    since $t^J \in R^J$. Thus, $t^{I,J}$ can not exist, and the case when neither $I \subset J$ nor $J \subset I$ can not arise.

    Finally, we observe that the previous analysis remains valid when dropping the positivity assumption on the weights, with the only exception being \eqref{eq:App_bp}. In particular, for negative weights the behaviour demonstrated in Example \ref{example:DPWL} may arise so that \eqref{eq:App_bp} is not valid in general anymore. However, by restricting the feasible weights, i.e., prohibiting negative weights of large magnitude via the condition in \eqref{eq:contCondv}, we ensure that \eqref{eq:App_bp} holds and the statement follows.
\end{proof}
\begin{remark}
    The observations about the parameter $\delta$ in Remark \ref{rmk:delta} directly transfer from the two- to the $d$-dimensional setting. Additionally, note that \eqref{eq:contCondv} is a necessary and sufficient condition for the firing time of an LSNN neuron to be continuous in the input firing times. However, the corresponding condition for a whole multi-layer LSNN as provided in \eqref{eq:SNN_CPWL} is sufficient but not necessary.
\end{remark}
We want to highlight some similarities and differences between two- and $d$-dimensional inputs. In both cases, the actual number of linear regions depends on the choice of parameter, in particular, the synaptic weights. Thereby, the number of linear regions scales at most as $2^d-1$ in the input dimension $d$ of an LSNN neuron, and the number is indeed attained if all weights are positive (as observed in Lemma \ref{prop:Partition}). However, the $d$-dimensional case allows for more flexibility in the structure of the linear regions. Recall that in the two-dimensional case, the boundary of any linear region is described by hyperplanes of the form \eqref{eq:App_Boundary2}. This does not hold if $d>2$, see e.g. \eqref{eq:App_LinReg}. Here, the weights also affect the shape of the linear region. Refining the connection between the boundaries of a linear region, its response function, and the specific choice of parameter requires further considerations. Similarly, obtaining a non-trivial upper bound on the number of linear regions for networks of LSNN neurons is not straightforward as the following example shows. 

\begin{example}
  Let $\Phi$ be a one-layer LSNN with $d_{\text{in}}$ input neurons and $d_{\text{out}}$ output neurons.
  Via Proposition \ref{prop:Partition}, we certainly can upper bound the number of linear regions generated by $\Phi$ by $(2^{d_{\text{in}}} - 1)^{d_{\text{out}}}$, i.e., the product of the number of linear regions generated by each output neuron. Unfortunately, the bound is far from optimal. Consider the case when $d_{\text{in}} = d_{\text{out}} = 2$. Then, the structure of the linear regions generated by the individual output neurons is given in \eqref{eq:App_Boundary2}. In particular, the boundary of the linear regions is described by a set of specific hyperplanes with a common normal vector, where the parameters of $\Phi$ only induce a shift of the hyperplanes. In other words, the hyperplanes separating the linear regions are parallel. Hence, each output neuron generates at most two parallel hyperplanes yielding three linear regions (see Figure \ref{fig:linRegEx}). The number of linear regions generated by $\Phi$ is therefore given by the number of regions four parallel hyperplanes can decompose the input domain into, i.e., at most $5 < 9 = (2^{d_{\text{in}}} - 1)^{d_{\text{out}}}$.      
\end{example}

Finally, under even stronger conditions, i.e., all weights are positive, we can further characterize the firing map of an LSNN neuron.
\begin{lemma}\label{lemma:Neuronproperties}
    Let $v$ be a LSNN neuron with with threshold $\theta_v >0$ and presynaptic neurons $u_1, \dots, u_d$ with corresponding weights $w_{u_i v} >0$, delays $d_{u_i v}\geq 0$, and firing times $t_{u_i} \in \R$, respectively. Then the firing time $t_v(t_{u_1},\dots,t_{u_d})$ as a function of the firing times $t_{u_1},\dots,t_{u_d}$ is a increasing and concave function. Moreover, the firing time of $v$ is given by
    \begin{equation}\label{eq:AlternativeFT}
        t_v(t_{u_1},\dots,t_{u_d}) = \inf_{\emptyset \neq I \subset [d]} \Big\{ s^I = \frac{1}{\sum_{i \in I} w_{u_i v}}\big(\theta_v + \sum_{i \in I}w_{u_i v} (t_{u_i} + d_{u_i v})\big) : s^I > \max_{i\in I} t_{u_i}   \Big\}.
    \end{equation}
\end{lemma}
\begin{proof}
    First, we show that the expression 
    \begin{equation*}
        s^\ast(t_{u_1},\dots,t_{u_d}) = \inf_{\emptyset \neq I \subset [d]} \Big\{ s^I = \frac{1}{\sum_{i \in I} w_{u_i v}}\big(\theta_v + \sum_{i \in I}w_{u_i v} (t_{u_i} + d_{u_i v})\big) : s^I > \max_{i\in I} t_{u_i}   \Big\}
    \end{equation*}
    is equivalent to the firing time $t_v(t_{u_1},\dots,t_{u_d})$ of an LSNN neuron $v$. By \eqref{eqn:firing_time_specific} we immediately observe that 
    \begin{equation*}
        t_v(t_{u_1},\dots,t_{u_d}) =  \frac{1}{\sum_{i \in F} w_{u_i v}}\big(\theta_v + \sum_{i \in F}w_{u_i v} (t_{u_i} + d_{u_i v})\big)   
    \end{equation*} 
    for some $F\subset [d]$ such that 
    \begin{equation*}
        \max_{i\in F} t_{u_i} < t_v(t_{u_1},\dots,t_{u_d}) \leq \min_{i\in [d]\setminus F} t_{u_i},    
    \end{equation*}
    i.e., $t_v(t_{u_1},\dots,t_{u_d}) \geq s^\ast(t_{u_1},\dots,t_{u_d})$. Moreover, set 
    \begin{equation*}
        s^J =  \frac{1}{\sum_{i \in J} w_{u_i v}}\big(\theta_v + \sum_{i \in J}w_{u_i v} (t_{u_i} + d_{u_i v})\big)  \quad \text{ for } J \subset [d] 
    \end{equation*}     
    and assume that $s^J > t_{u_j}$ for all $j\in J$. If $J\cap F^C \neq \emptyset$, then for some $k \in J\cap F^C$ we have  $s^J > t_{u_k}\geq t_v(t_{u_1},\dots,t_{u_d})$. Hence, assume that $J\cap F^C = \emptyset$, i.e., $J \subset F$. Then, either $s^J = t_v(t_{u_1},\dots,t_{u_d})$ (if $F=J$) or (for $F\neq J$) we get
    \begin{align*}
        &\sum_{i \in J}w_{u_i v} (s^J - t_{u_i} - d_{u_i v}) = \theta = \sum_{i \in F}w_{u_i v} (t_v(t_{u_1},\dots,t_{u_d}) - t_{u_i} - d_{u_i v}) \\
        &= \sum_{i \in J}w_{u_i v} (t_v(t_{u_1},\dots,t_{u_d}) - t_{u_i} - d_{u_i v}) + \sum_{i \in F\setminus J}w_{u_i v} (t_v(t_{u_1},\dots,t_{u_d}) - t_{u_i} - d_{u_i v})\\
        &> \sum_{i \in J}w_{u_i v} (t_v(t_{u_1},\dots,t_{u_d}) - t_{u_i} - d_{u_i v}), \quad \text{ i.e., } s^J > t_v(t_{u_1},\dots,t_{u_d}).
    \end{align*}
    Therefore $s^J \geq  t_v(t_{u_1},\dots,t_{u_d})$ so that $t_v(t_{u_1},\dots,t_{u_d}) \leq s^\ast(t_{u_1},\dots,t_{u_d})$ and \eqref{eq:AlternativeFT} follows.

    Next, we show that $t_v$ is an increasing function. Let $\varepsilon \in \R^d$ such that $\varepsilon_i \geq 0$ for all $i=1,\dots,d$. Denote the potential of $v$ for inputs $(t_{u_1},\dots,t_{u_d})$ and $(t_{u_1},\dots,t_{u_d}) + \varepsilon$ by $P_v$ and $P_v^\varepsilon$, respectively. Due to \eqref{eqn:firing_time}, we know that
    \begin{align*}
        \theta_v > P_v(t) &= \sum_{i \in [d]} w_{u_i v} \sigma(t - t_{u_i} - d_{u_i v}) \geq \sum_{i \in [d]} w_{u_i v} \sigma(t - t_{u_i} - \varepsilon_i - d_{u_i v}) \\
        &= P^\varepsilon_v(t) \quad \text{ for any } t < t_v(t_{u_1},\dots,t_{u_d}).
    \end{align*}
    Hence, the potential $P^\varepsilon_v(t)$ does not reach $\theta_v$ for $t <t_v(t_{u_1},\dots,t_{u_d})$. In other words, $t_v(t_{u_1}+\varepsilon_i,\dots,t_{u_d}+\varepsilon_d) \geq t_v(t_{u_1},\dots,t_{u_d})$ so that $t_v$ is indeed an increasing function.

    Finally, we prove the concavity of $t_v$. Let $x,y \in \R^d$ denote two distinct firing times of $u_1, \dots, u_d$. Our goal is to show that
    \begin{equation}\label{eq:ConcaveFct}
        t_v(p x + (1-p) y) \geq p t_v(x) + (1-p) t_v(y) \quad \text{ for all } 0 < p < 1.
    \end{equation}
    We already know that $t_v$ is a CPWL function, i.e., $t_v$ partitions $\R^d$ into linear regions. Therefore, we consider the following possibilities: Either $x$ and $y$ are in the same or different linear regions. Since the linear regions are defined by intersections of halfspaces (see \eqref{eq:App_FT} and \eqref{eq_App_FT2}), they are convex so that in the former case it is immediate to verify that \eqref{eq:ConcaveFct} holds. In the latter case, we need to further distinguish between two cases. Assume $x$ and $y$ are located in distinct linear regions $R^I$ and $R^J$, respectively, whereby $I,J \subset [d]$ indicate the subsets of neurons $u_1, \dots, u_d$ that trigger the firing of v on the given linear region. Then, either $p x + (1-p) y \in R^I \cup R^J$ or $p x + (1-p) y \in (R^I \cup R^J)^C$ for $0 < p < 1$. We will explicitly show the concavity of $t_v$ for $p x + (1-p) y \in R^I \cup R^J$, the other case follows along similar lines. Without loss of generality, we assume that $p x + (1-p) y \in R^I$. Moreover, we denote the restriction of $t_v$ to $R^I$ and $R^J$ by $A^I$ and $A^J$, respectively. Since $t_v$ is a CPWL function, $A^I$ and $A^J$ are affine functions so that $A^I(z^I) = (m^I)^T z^I + b^I$ and $A^J(z^J) = (m^J)^T z^J + b^J$ for $z^I \in R^I$, $z^J \in R^J$ and suitable parameter $m^I, m^J \in \R^d$, $b^I,b^J \in \R$. Using that $A^I$ and $A^J$ are affine, one derives that
    \begin{equation*}
        A^I(p x + (1-p) y) \geq p A^I(x) + (1-p) A^J(y) 
    \end{equation*}
    is equivalent to
    \begin{equation}\label{eq:concaveCond}
        (m^I - m^J)^T y + (b^I - b^J) \geq 0, 
    \end{equation}
    i.e., verifying \eqref{eq:concaveCond} suffices to obtain concavity of $t_v$. To that end, observe that $m^I, b^I$ and $m^J, b^J$ are determined by $I$ and $J$, respectively, i.e., by the weights and delays associated with $I$ and $J$ via the firing time of $v$ in \eqref{eqn:firing_time_specific}. In particular, we have
    \begin{equation*}
        m^I_i = 
        \begin{cases}
            0, &\text{ if } i \notin I \\
            \frac{w_{u_i v} }{\sum_{j\in I} w_{u_j v}}, &\text{ if } i \in I
        \end{cases}
        \quad 
        \text{ and }
        \quad
        b^I = \frac{\theta_v + \sum_{i\in I} w_{u_i v} d_{u_i v }}{\sum_{j\in I} w_{u_j v} }  ,      
    \end{equation*}
    and the analogous expressions for $m^J, b^J$. Using these expressions and the properties of the firing time $t_v$ one indeed obtains \eqref{eq:concaveCond}.
\end{proof}

\subsubsection{Networks of spiking neurons}

Next, we want to extend the properties of an LSNN neuron to a network of LSNN neurons. Due to the layer-wise arrangement of neurons in an LSNN, this task is rather straightforward. Thus, the results in Theorem \ref{thm:CPWL_mapping} and Proposition \ref{prop:SNNproperties} follow.
\begin{proof}[\textbf{of Theorem \ref{thm:CPWL_mapping}}] 
    In Lemma \ref{lemma:pwlNeuron}, we showed that the firing time of an LSNN neuron with arbitrarily many input neurons is a PWL function in the input firing times. Since $\Phi$ consists of LSNN neurons arranged in layers it immediately follows that the firing map of each layer is PWL. Thus, as a composition of PWL mappings  $t_\Phi$ itself is PWL. Furthermore, \eqref{eq:SNN_CPWL} guarantees that \eqref{eq:contCondv} holds for each neuron in $\Phi$, i.e., the firing time of each LSNN neuron continuously depends on its input firing times. Therefore, we conclude analogously as for the PWL property that $t_\Phi$ is continuous, i.e., $t_\Phi$ is a CPWL mapping, if \eqref{eq:SNN_CPWL} is satisfied.    
\end{proof}
As already indicated, the condition in Theorem \ref{thm:CPWL_mapping} is not necessary to obtain a continuous firing map. Next, we specifically construct an LSNN demonstrating this observation.
\begin{example}
    Consider a two-layer LSNN $\Phi$ defined by
    \begin{equation*}\label{eq:SNN_counter_example_negative_CPWL}
        W^1 = \begin{pmatrix}
            1 & 0 & 0 \\
            0 & 1 & \frac{1}{2} 
        \end{pmatrix}, 
        W^2 = \begin{pmatrix}
            1\\
            1\\
            -1
        \end{pmatrix},
        \Theta^1 = \begin{pmatrix}
            1 \\
            1 \\
            1
        \end{pmatrix}, 
        \Theta^2 = 1.
    \end{equation*}
    Without loss of generality, we assume the delays to be zero. We show that despite the negative weight, which violates \eqref{eq:SNN_CPWL}, $t_{\Phi}(t_1, t_2)$ is a CPWL function for the input firing times $t_{u_1}, t_{u_2} \in \R$. Denote the input neurons by $u_1, u_2$ and the neurons in the hidden layer by $v_1, v_2, v_3$. Note that the firing times of the neurons in the hidden layer depend on either $u_1$ or $u_2$. Via \eqref{eqn:firing_time_specific}, the firing times of the neurons in the hidden layer are 
    \begin{equation*}
        t_{v_1} = 1+t_{u_1}, \quad t_{v_2} = 1+t_{u_2}, \quad t_{v_3} = 2+t_{u_2}. 
    \end{equation*}
Similarly, via \eqref{eqn:firing_time_specific}, the firing time of the output neuron is
    \begin{equation*}
        t_\Phi(t_{u_1}, t_{u_2}) = \begin{cases}
            2 + t_{u_1},  &\text{ if } t_{v_1} < t_{v_2} \\
            2 + t_{u_2},  &\text{ if } t_{v_2} < t_{v_1} \\ 
            \frac{3}{2} + \frac{1}{2}(t_{u_1}+t_{u_2}),  &\text{ if } \abs{t_{v_1} - t_{v_2}} \leq 1
        \end{cases}.
    \end{equation*}
Hence, the breakpoints $(t^b_{u_1} - t^b_{u_2})\in \R^2$ of the linear regions are determined by 
\begin{equation*}
    \abs{t^b_{u_1} - t^b_{u_2}} = 1 \Leftrightarrow t^b_{u_1} = 
    \begin{cases}
            t^b_{u_2} + 1, &\text{ if } t^b_{u_1} > t^b_{u_2} \\
            t^b_{u_2} - 1, &\text{ if } t^b_{u_1} < t^b_{u_2} 
        \end{cases}.        
\end{equation*}
Evaluating the firing time of the output neuron at the breakpoints gives
\begin{equation*}
        t_\Phi(t^b_{u_1}, t^b_{u_2}) = \begin{cases}
            2 + t^b_{u_2},  &\text{ if } t^b_{u_1} > t^b_{u_2} \\
            2 + t^b_{u_1},  &\text{ if } t^b_{u_1} < t^b_{u_2}
        \end{cases},
\end{equation*}
which shows that $t_\Phi$ is indeed continuous.
\end{example}

With the same arguments as in the proof of Theorem \ref{thm:CPWL_mapping}, we also show Proposition \ref{prop:SNNproperties}. 
\begin{proof}[\textbf{of Proposition \ref{prop:SNNproperties}}] 
    First, the alternative expression for the firing time of an LSNN neuron with only positive weights is proven in Lemma \ref{lemma:Neuronproperties}. Hence, it is left to show that $t_\Phi$ is increasing and concave under the given condition. However, this is a direct consequence of Lemma \ref{lemma:Neuronproperties} and the structure of $\Phi$. In particular, one establishes that $t_\Phi$ is increasing as a composition of strictly increasing maps and concludes that $t_\Phi$ is also concave since the composition of non-decreasing and concave functions is again concave.
\end{proof}
Finally, we construct specific LSNNs with locally decreasing and non-concave firing maps, highlighting the need for the positive weights condition.
\begin{example}
        Let $v$ be an LSNN neuron with threshold $\theta_v = 1$ and presynaptic neurons $u_1, u_2, u_3$ with corresponding weights $w_{u_1 v} = -\frac{1}{2}, w_{u_2 v} = 1, w_{u_3 v} = 1$, delays $d_{u_1 v} = d_{u_2 v} =  d_{u_3 v} = 0$. Additionally, let $t_a = (0,1,1)^T$ and $t_b = (\frac{1}{2},1,1)^T$ be two data points corresponding to the firing times of input neurons. 
        Via \eqref{eqn:firing_time_specific}, one can directly verify that  $t_v(t_a) > t_v(t_b)$ when $t_a < t_b$, i.e., $t_v$ is not increasing. For the non-concave part, let $t_x = (\frac{1}{2},0,0)^T$ and $t_y = (\frac{1}{2},1,1)^T$ as two data points corresponding to the firing times of input neurons. Setting $p = \frac{1}{2}$ and again via \eqref{eqn:firing_time_specific}, one can directly verify that 
        \begin{equation*}
            pt_v(t_x, t_y) + (1-p)t_v(t_x, t_y) \geq t_v(pt_x + (1-p)t_y).
        \end{equation*}
\end{example}

\subsection{Realizing ReLU with LSNNs}
\label{sec:RealizingReLU_app}

In this section, we show that LSNNs can realize the ReLU function. To that end, we first analyze the ability of an LSNN neuron to express certain (simple) piecewise functions via its firing map.

\begin{proposition}\label{prop:f_1}
    Let $c_1, c_2 \in \R$,  $c_3 \in (a,b) \subset \R$ and consider $f_1,f_2: [a,b] \to \R$ given by 
    \begin{equation*}
        f_1(x) = \begin{cases}
            x + c_1 &, \text{ if } x > c_3 \\
            c_2 &, \text{ if } x < c_3
        \end{cases}
        \quad
        \text{ or }
        \quad
        f_2(x) = \begin{cases}
            x + c_1 &, \text{ if } x < c_3 \\
            c_2 &, \text{ if } x > c_3
        \end{cases},
    \end{equation*}
    where we do not fix the value at the breakpoint $x=c_3$.
    There does not exist an LSNN neuron $v$ with input neuron $u_1$ such that $t_v(x) = f(x)$  on $[a,b]$, where $f \in \{f_1, -f_1, f_2\}$ and $t_v(x)$ denotes the firing time of $v$ on input $t_{u_1} = x$.
\end{proposition}
\begin{proof}
    If $u_1$ is the only input neuron, then $v$ fires if and only if $w_{u_1 v} > 0$ and by \eqref{eqn:firing_time_specific} the firing time is given by
    \begin{equation*}
            t_v(x) = \frac{\theta}{w_{u_1 v}} + x + d_{u_1 v}  \quad \text{ for all } x \in [a,b],   
    \end{equation*}
    i.e., $t_v \neq f$. Therefore, we introduce auxiliary input neurons $u_2, \dots, u_n$ and assume without loss of generality that $t_{u_i} + d_{u_i v} < t_{u_j} + d_{u_j v}$ for $j >i$. Here, the firing times $t_{u_i}$, $i=2, \dots,n$, are suitable constants. We will show that even in this extended setting $t_v \neq f$ still holds and thereby also the claim.

    For the sake of contradiction, assume that $t_v(x) = f_1(x)$ for all $x \in [a,b]$. This implies that there exists an index set $J \subset \{1,\dots,n\}$ with $\sum_{j\in J} w_{u_j v} > 0$ and a corresponding non-empty interval $(a_1,c_3) \subset [a,b]$ such that 
    \begin{equation}\label{eq:const_part}
        c_2 = t_v(x) =  \frac{1}{\sum_{i\in J} w_{u_i v}} \Big(\theta_v + \sum_{i\in J} w_{u_i v} (t_{u_i} + d_{u_i v})\Big)  \quad \text{ for all } x \in (a_1,c_3),
    \end{equation}
    where we have applied \eqref{eqn:firing_time_specific}. Note that $J$ is of the form $J = \{1,\dots, \ell\}$ for some $\ell \in \{2,\dots,n\}$ because $(t_{u_i} + d_{u_i v})_{i=2}^n$ is in ascending order, i.e., if the spike from $u_\ell$ has reached $v$ before $v$ fired, then so did the spikes from $u_i$, $2\leq i <\ell$. In particular, we immediately observe that $1 \notin J$ since otherwise, due to the contribution from $u_1$, $t_v$ is non-constant on $(a_1,c_3)$, i.e., $t_v \neq c_3$ on $(a_1,c_3)$. Hence, the spike from $u_1$ with firing time $x \in (a_1,c_3)$ necessarily reaches $v$ after the subset of neurons specified by $J$ already caused $v$ to fire. Therefore, we have
    \begin{equation*}
        x + d_{u_1 v} \geq t_v(x) = c_2 \quad \text{ for all } x \in (a_1,c_3).
    \end{equation*}
    However, it immediately follows for any $y \geq c_3$ that 
    \begin{equation*}
        y + d_{u_1 v} > x + d_{u_1 v} \geq c_2 \quad \text{ for all } x \in (a_1,c_3). 
    \end{equation*}
    Thus, we know that a spike from $u_1$ with $t_{u_1} > c_3$ does not reach $v$ before it emits a spike unless there are additional incoming spikes in $v$ from neurons $\{u_i: i \in I \subset \{2,\dots,n\}\setminus J= \{\ell+1,\dots,n\}\}$, which delay its firing. 
    By the same reasoning as before, we conclude that $t_{u_{\ell+1}} + d_{u_{\ell+1} v} \geq c_2$ because otherwise \eqref{eq:const_part} would be violated since spikes from neurons $\{u_i: i \in I \}$ contribute to the firing of $v$ on $(a_1,c_3)$. Consequently, the firing of $v$ can not be delayed, i.e., 
    $t_v(x) = c_2$ for all $x\in (a_1,b]$, which contradicts $t_v = f_1$ on $[a,b]$. The same reasoning can be applied to derive that $-f_1$ can not be emulated by the firing map of an LSNN neuron. 
    
    We perform a similar analysis to show that $f_2$ can not be emulated. For the sake of contradiction, assume that $t_v(x) = f_2(x)$ for all $x \in [a,b]$. This implies that there exists an index set $I \subset \{1,\dots,n\}$ with $\sum_{i\in I} w_{u_i v} > 0$ and a corresponding non-empty interval $(a_1,c_3) \subset [a,b]$ such that
    \begin{equation}\label{eq:non-const}
        x + c_1 = t_v(x) =  \frac{1}{\sum_{i\in I} w_{u_i v}} \Big(\theta_v +  w_{u_1 v} (x + d_{u_1 v}) + \sum_{i\in I\setminus\{1\}} w_{u_i v} (t_{u_i} + d_{u_i v})\Big) \quad \text{ for } x \in (a_1,c_3),
    \end{equation}
    where we have applied \eqref{eqn:firing_time_specific}. Note that $1 \in I$ necessarily needs to hold, since otherwise $t_v$ is constant on $(a_1,c_3)$. 
    Hence, $I$ is of the form $I = \{1,\dots, \ell\}$ for some $\ell \in \{1,\dots,n\}$. To satisfy $t_v(x) = f_2(x)$ for all $x \in [a,b]$, there additionally needs to exist an index set $J \subset \{1,\dots,n\}$ with $\sum_{j\in J} w_{u_j v} > 0$ and a corresponding non-empty interval $(c_3,b_1) \subset [a,b]$ such that $t_v = c_2$ on $(c_3,b_1)$. We conclude that $J = \{1,\dots, m\}$ or $J = \{2,\dots,m\}$ for some $m \in \{1,\dots,n\}$. In the former case, $t_v$ is non-constant on $(c_3,b_1)$ (due to the contribution from $u_1$), i.e., $t_v \neq c_2$ on $(c_3,b_1)$. Hence, it remains to consider the latter case. First, note that $c_1 \leq 0$ is not a admissible choice since \eqref{eq:non-const} implies that for $x \in (a_1,c_3)$
    \begin{equation*}
        t_v(x) = x + c_1 \leq x, 
    \end{equation*}
    i.e., the spike from $u_1$ does not reach $v$ before its firing, which contradicts the construction. Thus, there exists some $0< \varepsilon < c_1$ such that $c_3 - \varepsilon \in (a_1,c_3)$ and
    \begin{equation*}
        t_v(c_3 - \varepsilon) = c_3 - \varepsilon + c_1 > c_3.
    \end{equation*}
    This particularly entails that no spikes from neurons $\{u_j : j \in J\setminus I\}$ reach $v$ before time $c_3 - \varepsilon + c_1$, i.e., $t_{u_{\ell+1}} +d_{u_{\ell+1}} \geq c_3 - \varepsilon + c_1 > c_3$. Therefore, we derive that $m \leq \ell$ needs to hold since otherwise $u_1$ with $t_{u_1}= x \in (c_3, c_3 - \varepsilon + c_1)$ contributes to the firing of $v$, contradicting the fact that $t_v$ is constant on $(c_3, b_1)$. Moreover, $m= \ell$, i.e., $J = I\setminus\{1\}$, is not valid because either $J$ is empty or by comparing the coefficients of $x$ in \eqref{eq:non-const} we find that 
    \begin{equation}\label{eq:weight_constraint}
        \frac{w_{u_1 v}}{\sum_{i\in I} w_{u_i v}} = 1 \Leftrightarrow \sum_{i\in I\setminus\{1\}} w_{u_i v} = 0, 
    \end{equation}
    which implies the contradiction
    \begin{equation*}
        0 = \sum_{i\in I\setminus\{1\}} w_{u_i v} = \sum_{j\in J} w_{u_j v} >0.
    \end{equation*}
    Finally, $m< \ell$ is also not feasible. This follows from \eqref{eq:weight_constraint} via the observation that $w_{u_1 v}>0$, i.e., the contribution from $u_1$ to the potential of $v$ is necessarily positive if the associated spike arrives in time. Consequently, when $u_1$ stops contributing to the firing of $v$ on $(c_3,b_1)$ the firing time $t_v$ did not decrease (since by the assumption $m< \ell$ also no further incoming spikes arrive). Hence, the spikes from $u_2, \dots, u_{\ell}$ contribute to the firing of $v$ on any input $x\in (c_3,b_1)$ contradicting $m< \ell$. Thus, $J$ can not exist and thereby $t_v = f_2$ on $[a,b]$ can not be achieved.
\end{proof}
Next, we show that $-f_2$ as defined in Proposition \ref{prop:f_1} can indeed be emulated by the firing map of an LSNN neuron under suitable conditions. 
\begin{proposition}
\label{prop:4beforerelu}
    Let $c_1, c_2 \in \R$,  $c_3 \in (a,b) \subset \R$ and consider $f: [a,b] \to \R$
    \begin{equation*}
        f(x) = \begin{cases}
            - x + c_1 &, \text{ if } x < c_3 \\
            c_2 &, \text{ if } x > c_3
        \end{cases},
    \end{equation*}
    where we do not fix the value at the breakpoint $x=c_3$. 
    There exists a one-layer LSNN $\Phi$ with output neuron $v$ and input neuron $u_1$ such that $t_v(x) = f(x)$ on $[a,b]$, where $t_v(x)$ denotes the firing time of $v$ on input $t_{u_1} = x$, if and only if $c_1 - c_3 = c_2$ as well as $c_1 \geq 2 c_3$.    
\end{proposition}
\begin{proof}
    First, we show that $t_v$ can not emulate $f$ if the conditions $c_1 - c_3 = c_2$ and $c_1 \geq 2 c_3$ are not met. The argument is essentially the same as in the proof of Proposition \ref{prop:f_1} and we only sketch the main steps. Assuming that $t_v(x) = f(x)$ for all $x \in [a,b]$ implies that there exists an index set $I = \{1\} \cup \{2,\dots, \ell\}$ for some $\ell \in \{2,\dots,n\}$ with $\sum_{i\in I} w_{u_i v} > 0$ and a corresponding non-empty interval $(a_1,c_3) \subset [a,b]$ such that
    \begin{equation}\label{eq:-f_2}
        - x + c_1 = t_v(x) =  \frac{1}{\sum_{i\in I} w_{u_i v}} \Big(\theta_v +  w_{u_1 v} (x + d_{u_1 v}) + \sum_{i\in I\setminus\{1\}} w_{u_i v} (t_{u_i} + d_{u_i v})\Big) \quad \text{ for } x \in (a_1,c_3),
    \end{equation}
    where we have applied \eqref{eqn:firing_time_specific}. Due to $1 \in I$, we obtain 
    \begin{equation}\label{eq:-f_2const}
        - x + c_1 = t_v(x) > x + d_{u_1 v} \Leftrightarrow c_1 > 2x + d_{u_1 v}  \quad \text{ for } x \in (a_1,c_3),    
    \end{equation}
    and in particular $c_1 \geq 2 c_3$ needs to be satisfied. Hence, for $\varepsilon > 0$ such that $c_3 - \varepsilon \in (a_1,c_3)$ we find via \eqref{eq:-f_2} that
    \begin{equation}\label{eq:-f_2spike}
        t_v(c_3 - \varepsilon) = - c_3 + \varepsilon + c_1 \geq  c_3 + \varepsilon > c_3.
    \end{equation}
    Moreover, to satisfy $t_v(x) = f_2(x)$ for all $x \in [a,b]$ to hold, there additionally needs to exist an index set $J = \{2,\dots,m\}$ for some $m \in \{2,\dots,n\}$ with $\sum_{j\in J} w_{u_j v} > 0$ and a corresponding non-empty interval $(c_3,b_1) \subset [a,b]$ such that $t_v = c_2$ on $(c_3,b_1)$. Due to \eqref{eq:-f_2spike} we conclude that no spikes from neurons $\{u_j : j \in J\setminus I\}$ reach $v$ before time $c_3$ so that $m \leq \ell$ needs to be satisfied. Finally, assuming that $c_1 - c_3 \neq c_2$ entails that there is a jump  at $c_3$ in $f$, i.e., $f$ is discontinuous at $c_3$. However, a jump discontinuity requires an incoming spike that delays the firing of $v$ for $t_{u_1} = c_3$, which we already excluded. Thus, $t_v(x) = f(x)$ for all $x \in [a,b]$ can not be achieved given that $c_1 - c_3 \neq c_2$.

    For the reverse direction, we will explicitly construct an LSNN neuron $v$ that emulates $f$ if the conditions $c_1 - c_3 = c_2$ and $c_1 \geq 2 c_3$ are fulfilled.
    We introduce an auxiliary input neuron with constant firing time $t_{u_2} \in \R$ and specify the parameter of $\Phi = ((W, D, \Theta))$ in the following manner (see Figure \ref{subfig:OneLayerSNN}):
    \begin{equation*}
        W = \begin{pmatrix}
            -\frac{1}{2} \\
            1
        \end{pmatrix},  
        D = \begin{pmatrix}
            d_1  \\
            d_2
        \end{pmatrix},
        \Theta = \theta, 
    \end{equation*}
    where $\theta,d_1,d_2 >0$ are to be specified. Note that either $u_2$ or $u_1$ together with $u_2$ can trigger a spike in $v$ since $w_{u_1 v} < 0$. Therefore, applying \eqref{eqn:firing_time_specific} yields that $u_2$ triggers a spike in $v$ under the following circumstances:
    \begin{equation*}
        t_{v}(x) = \theta  + t_{u_2} + d_2  \quad \text{ if } t_{v}(x) \leq  t_{u_1} + d_1 = x + d_1.
    \end{equation*}
    Hence, this case only arises when
    \begin{equation*}
        \theta  + t_{u_2} + d_2 \leq  x + d_1 \Leftrightarrow \theta  + t_{u_2} + d_2 - d_1  \leq  x. 
    \end{equation*}    
    To emulate $f$ the parameters need to satisfy 
    \begin{equation*}
        \theta  + t_{u_2} + d_2 - d_1 \leq x \text{ for all } x \in (c_3,b] \quad \text{ and } \quad \theta  + t_{u_2} + d_2 - d_1 > x \text{ for all } x \in [a,c_3)   
    \end{equation*}
    which simplifies to 
    \begin{equation}\label{eq:propCond1}
        \theta  + t_{u_2} + d_2 - d_1 = c_3.
    \end{equation}
    If the additional condition
    \begin{equation}\label{eq:propCond2}
        \theta  + t_{u_2} + d_2 = c_2      
    \end{equation}
    is met, we can infer that for $d_1=c_2 - c_3$ (which is a valid choice due to $c_2 - c_3 = c_1 - 2c_3\geq 2 c_3 - 2 c_3=0$)
    \begin{equation*}
        t_v(x) = \begin{cases}
            2 (\theta + t_{u_2} + d_2) - (x + d_1) &, \text{ if } x < c_3 \\
            \theta  + t_{u_2} + d_2 &, \text{ if } x \geq c_3
        \end{cases} = 
        \begin{cases}
            - x + c_1 &, \text{ if } x < 0 \\
            c_2 &, \text{ if } x \geq 0
        \end{cases}.
    \end{equation*}
    Finally, it is immediate to verify that the conditions \eqref{eq:propCond1} and \eqref{eq:propCond2} can be satisfied simultaneously by choosing $d_2 = d_1$ and $t_{u_2} = c_3 - \theta$.
\end{proof}

Having established the capability or inability to emulate certain simple affine functions by the firing map of an LSNN neuron, we now turn to their realizations. To that end, recall that to realize a function $f:[a,b] \to \R$ we focus on encoding schemes of the type $T_{\text{in}} \slash T_{\text{out}} \pm \cdot$. Therefore, for an LSNN neuron $v$ with input neuron $u_1$ and $x\in [a,b]$ we distinguish the following encoding schemes:
\begin{align*}
    \text{a)}& \quad
    \begin{cases}
        t_{u_1} &= T_{\text{in}} + x =: z \\
        t_v(z) &= T_{\text{out}} + \mathcal{R}_v(x) \Leftrightarrow \mathcal{R}_v(x) = - T_{\text{out}} + t_v(z)    
    \end{cases}
    \\
    \text{b)}& \quad
    \begin{cases}
        t_{u_1} &= T_{\text{in}} - x =: z \\
        t_v(z) &= T_{\text{out}} - \mathcal{R}_v(x) \Leftrightarrow \mathcal{R}_v(x) =  T_{\text{out}} -t_v(z)    
    \end{cases}      
    \\
    \text{c)}& \quad
    \begin{cases}
        t_{u_1} &= T_{\text{in}} - x =: z \\
        t_v(z) &= T_{\text{out}} + \mathcal{R}_v(x) \Leftrightarrow \mathcal{R}_v(x) = - T_{\text{out}} + t_v(z)    
    \end{cases}   
    \\
    \text{d)}& \quad
    \begin{cases}
        t_{u_1} &= T_{\text{in}} + x =: z \\
        t_v(z) &= T_{\text{out}} - \mathcal{R}_v(x) \Leftrightarrow \mathcal{R}_v(x) = T_{\text{out}} - t_v(z)    
    \end{cases}   
\end{align*}
Note that cases a) and b) describe consistent input and output encoding schemes, whereas c) and d) provide inconsistent in the sense that input and output formalism do not match. Consistency is an important property when stacking LSNNs because otherwise, the realization of the stacked network does not match the composition of the subnetworks in general; see Section \ref{section:SNN_calculus}. 

Now, consider the function $f_1:[a,b] \to \R$ given by
\begin{equation*}
    f_1(x) = \begin{cases}
        x + c_1 &, \text{ if } x > c_3 \\
        c_2 &, \text{ if } x < c_3
    \end{cases}
\end{equation*}
for some constants $c_1, c_2 \in \R$ and $c_3 \in (a,b)$. To realize $f_1$ by an LSNN neuron $v$ via scheme a) we need
\begin{equation*}
    f_1(x) = \mathcal{R}_v(x) = - T_{\text{out}} + t_v(z) \quad \text{ for all } x \in [a,b],
\end{equation*}
i.e.,
\begin{equation}\label{eq:fctfiringmap}
    t_v(z) = \begin{cases}
        z  + T_{\text{out}} - T_{\text{in}} + c_1 &, \text{ if } z - T_{\text{in}} > c_3 \\
        c_2 + T_{\text{out}} &, \text{ if } z - T_{\text{in}} < c_3
    \end{cases}
    \Leftrightarrow
    t_v(z) = \begin{cases}
        z  + c^\prime_1 &, \text{ if } z > c^\prime_3 \\
        c^\prime_2 &, \text{ if } z  < c^\prime_3
    \end{cases},    
\end{equation}
where $c^\prime_1 = T_{\text{out}} - T_{\text{in}} + c_1$, $c^\prime_2 =  c_2 + T_{\text{out}}$, and $c^\prime_3 = c_3 + T_{\text{in}}$. However, Proposition \ref{prop:f_1} implies that the sought function in \eqref{eq:fctfiringmap} can not be expressed as the firing map of $v$. Hence, $f_1$ can not be realized by an LSNN neuron with encoding scheme a). Similar observations can be made for the other encoding schemes and the realization of the functions in Proposition \ref{prop:f_1} and \ref{prop:4beforerelu}. Our goal is to realize the ReLU activation function and it is now straightforward to verify that ReLU can not be realized by an LSNN neuron with consistent encoding. In contrast, for the inconsistent scheme c) we obtain for a given input domain $[a,b]$
\begin{equation*}
    \sigma(x) = \mathcal{R}_v(x) = - T_{\text{out}} + t_v(z),
\end{equation*}
i.e.,
\begin{equation*}
    t_v(z) = \begin{cases}
        - z  + T_{\text{out}} + T_{\text{in}} &, \text{ if } z \leq T_{\text{in}} \\
        T_{\text{out}} &, \text{ if } z > T_{\text{in}}
    \end{cases},   
\end{equation*}
which can be expressed by the firing map of $v$ for a suitable choice of $T_{\text{out}}$ and $T_{\text{in}}$; see Proposition \ref{prop:4beforerelu}. We summarize these observations in our next result.
\begin{proposition}
    An LSNN neuron can not realize ReLU on a given domain with consistent encoding, whereas ReLU can be realized when applying an inconsistent encoding.
\end{proposition}
Is it possible to realize ReLU with a consistent encoding? To realize $\sigma$ by an LSNN $\Phi$ via scheme a) we need
\begin{equation}\label{eq:sigmaRealizing}
    \sigma(x) = \mathcal{R}_\Phi(x) = - T_{\text{out}} + t_v(z) \quad \text{ for all } x \in [a,b],
\end{equation}
i.e.,
\begin{equation}\label{eq:sigmaRealizingFiring}
    t_v(z) = 
        \begin{cases}
            z  + T_{\text{out}} - T_{\text{in}}  &, \text{ if } z - T_{\text{in}} > 0 \\
            T_{\text{out}} &, \text{ if } z - T_{\text{in}} \leq 0
        \end{cases}
    \Leftrightarrow
    t_v(z) = 
        \begin{cases}
            z  + T_{\text{out}} - T_{\text{in}}  &, \text{ if } z >  T_{\text{in}} \\
            T_{\text{out}} &, \text{ if } z \leq  T_{\text{in}}
        \end{cases}.
\end{equation}
We proceed by constructing a two-layer LSNN that indeed achieves this goal.

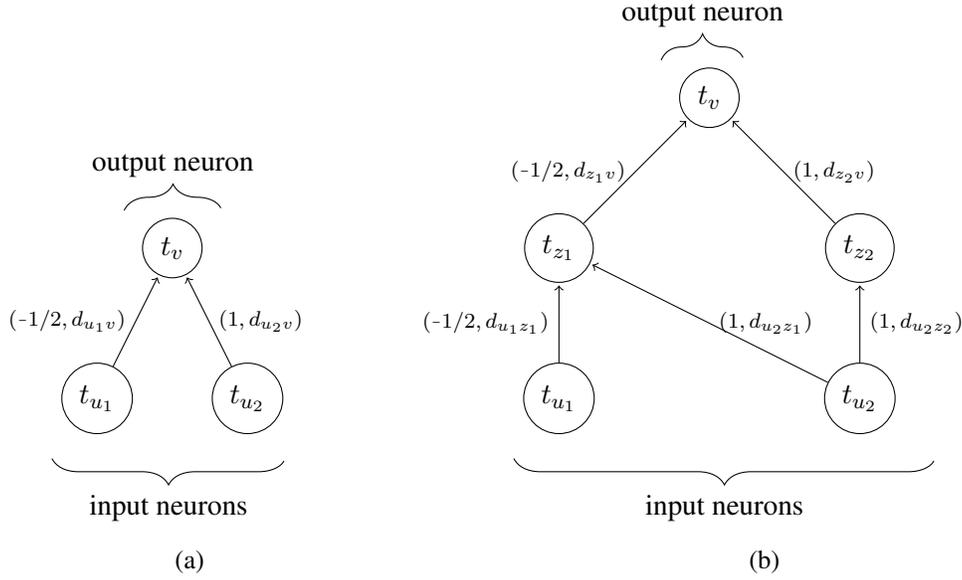
\begin{figure}[t]
	\centering
 \begin{subfigure}[b]{0.49\textwidth}
     \centering
     \hspace{-1cm}
     \begin{tikzpicture}[shorten >=1pt]
		\tikzstyle{unit}=[draw,shape=circle,minimum size=0.5cm]
		\tikzstyle{hidden}=[draw,shape=circle,fill=black!25,minimum size=0.5cm]
		\tikzstyle{hidden}=[draw,shape=circle,minimum size=0.5cm]

            \node[unit](x0_0) at (-7,3.5){$t_{u_1}$};
            \node[unit](x0_1) at (-5,3.5){$t_{u_2}$};
            \node[unit](o1) at (-6,5.5){$t_v$};            
            \draw[->] (x0_0) -- (o1) node[midway,left] {\scriptsize $(\shortminus1/2,d_{u_1v})$};
            \draw[->] (x0_1) -- (o1) node[midway,right] {\scriptsize $(1,d_{u_2v})$};

            \draw [decorate,decoration={brace,amplitude=10pt,mirror}](-7.6,2.7) -- (-4.5,2.7) node[midway,yshift=-1.7em]{input neurons};
            \draw [decorate,decoration={brace,amplitude=10pt},xshift=-4pt,yshift=0pt] (-6.5,6) -- (-5.2,6) node [black,midway,yshift=+0.6cm]{output neuron};
            \end{tikzpicture}
            \caption{}
            \label{subfig:OneLayerSNN}
 \end{subfigure}
 \begin{subfigure}[b]{0.5\textwidth}
     \centering
     \hspace{-2cm}
     \begin{tikzpicture}[shorten >=1pt]
		\tikzstyle{unit}=[draw,shape=circle,minimum size=0.5cm]
		\tikzstyle{hidden}=[draw,shape=circle,fill=black!25,minimum size=0.5cm]
		\tikzstyle{hidden}=[draw,shape=circle,minimum size=0.5cm]

            \node[unit](x0) at (0,3.5){$t_{u_1}$};
            \node[unit](xd) at (4,3.5){$t_{u_2}$};
            \node[hidden](h10) at (0,5.5){$t_{z_1}$}; 
		\node[hidden](h11) at (4,5.5){$t_{z_2}$}; 
            \node[unit](output) at (2,7.5){$t_v$}; 
            \draw[->] (x0) -- (h10) node[midway,left] {\scriptsize $(\shortminus 1/2,d_{u_1z_1})$};
            \draw[->] (xd) -- (h10) node[midway,right] {\scriptsize $(1,d_{u_2z_1})$};
		\draw[->] (xd) -- (h11) node[midway,right] {\scriptsize $(1,d_{u_2z_2})$};
            \draw[->] (h10) -- (output) node[midway,left] {\scriptsize $(\shortminus1/2,d_{z_1v})$};
            \draw[->] (h11) -- (output) node[midway,right] {\scriptsize $(1,d_{z_2v})$};
    \draw [decorate,decoration={brace,amplitude=10pt,mirror}](-0.6,2.7) -- (5,2.7) node[midway,yshift=-1.7em]{input neurons};
		\draw [decorate,decoration={brace,amplitude=10pt},xshift=-4pt,yshift=0pt] (1.5,8) -- (2.6,8) node [black,midway,yshift=+0.6cm]{output neuron};
  \end{tikzpicture}
  \caption{}
    \label{subfig:TwoLayerSNN}
     \end{subfigure}
	\caption{a) Computation graph associated with an LSNN with two input neurons and one output neuron that realizes $f$ as defined in Proposition \ref{prop:4beforerelu}. (b) Stacking the network in (a) twice results in an LSNN that realizes the ReLU activation function. }
\end{figure}

\begin{proposition}\label{prop:GenReLu}
     Let $c_1, c_2 \in \R$ be such that $c_2 > c_1$ and $c_1 + c_2 > 2 b$. Consider $f: [a,b] \to \R$ defined as 
    \begin{equation*}
        f(x) = \begin{cases}
            x + c_2 - c_1 &, \text{ if } x > c_1 \\
            c_2 &, \text{ if } x \leq c_1
        \end{cases}.
    \end{equation*}
    There exists a two-layer LSNN $\Phi$ with output neuron $v$ and input neuron $u_1$ such that $t_v(x) = f(x)$ on $[a,b]$, where $t_v(x)$ denotes the firing time of $v$ on input $t_{u_1} = x$.    
\end{proposition}
\begin{proof} 
    We introduce an auxiliary input neuron $u_2$ with constant firing time $t_{u_2} \in \R$  and specify the parameter of $\Phi = ((W^1, D^1, \Theta^1), (W^2, D^2, \Theta^2))$ in the following manner:
    \begin{equation}\label{eq:SNN_Relu_param}
        W^1 = \begin{pmatrix}
            -\frac{1}{2} & 0 \\
            1 & 1 
        \end{pmatrix},  
        D^1 = \begin{pmatrix}
            d & 0 \\
            d & d
        \end{pmatrix},
        \Theta^1 = \begin{pmatrix}
            \theta \\
            \theta
        \end{pmatrix}, 
        W^2= \begin{pmatrix}
            -\frac{1}{2}  \\
            1 
        \end{pmatrix},
        D^2 = \begin{pmatrix}
            d \\
            d
        \end{pmatrix},
        \Theta^2 = \theta,
    \end{equation}
    where $d \geq 0$ and $\theta > 0$ is chosen such that $\theta + t_{u_2}  > b$.
    We denote the input neurons by $u_1, u_2$, the neurons in the hidden layer by $z_1, z_2$, and the output neuron by $v$. Note that the firing time of $z_1$ depends on $u_1$ and $u_2$. In particular, either $u_2$ or $u_1$ together with $u_2$ can trigger a spike in $z_1$ since $w_{u_1 z_1} < 0$. Therefore, applying \eqref{eqn:firing_time} yields that $u_2$ triggers a spike in $z_1$ under the following circumstances:
    \begin{equation*}
        t_{z_1}(x) = \theta  + t_{u_2} + d  \quad \text{ if } t_{z_1}(x) \leq  t_{u_1} + d = x + d.
    \end{equation*}
    Hence, this case only arises when
    \begin{equation}\label{eq:propCond}
        \theta  + t_{u_2} + d \leq  x + d \Leftrightarrow \theta  + t_{u_2}  \leq  x. 
    \end{equation}    
    However, by construction $\theta  + t_{u_2} > b$, so that \eqref{eq:propCond} does not hold for any $x \in [a,b]$. Thus, we conclude via \eqref{eqn:firing_time_specific} that
    \begin{equation*}
        t_{z_1}(x) = 2 (\theta  + t_{u_2} + d) - (x +d) =  2 (\theta  + t_{u_2} ) + d - x .
    \end{equation*}
    By construction, the firing time $t_{z_2} = \theta + t_{u_2} + d$ of $z_2$ is a constant which depends on the input only via $u_2$. A similar analysis as in the first layer shows that  
    \begin{equation*}
        t_{v}(x) = \theta  + t_{z_2} + d  \quad \text{ if } t_{v}(x) \leq  t_{z_1} + d = 2 (\theta  + t_{u_2}) + d - x + d = 2 (\theta  + t_{u_2} + d) - x.
    \end{equation*}
    Hence, $z_2$ triggers a spike in $v$ when
    \begin{equation*}
        \theta + t_{z_2} + d = \theta  + \theta + t_{u_2} + d + d \leq 2 (\theta  + t_{u_2} +d) - x \quad  \Leftrightarrow \quad x  \leq  t_{u_2}. 
    \end{equation*}     
    If the additional condition
    \begin{equation}\label{eq:propCond22}
        \theta  + t_{z_2} + d =  c_2 \quad \Leftrightarrow \quad 2(\theta +d)  + t_{u_2} = c_2    
    \end{equation}
    is met, we can infer that 
    \begin{align*}
        t_v(x) &= \begin{cases}
            2 (\theta + t_{z_2} + d) - (t_{z_1}(x) + d) &, \text{ if } x  >  t_{u_2} \\
            \theta  + t_{z_2} + d &, \text{ if }  x  \leq  t_{u_2}
        \end{cases}\\
        &= 
        \begin{cases}
            2c_2 - (2 (\theta  + t_{u_2} ) + d - x + d) &, \text{ if } x  >  t_{u_2} \\
            c_2 &, \text{ if } x  \leq  t_{u_2}
        \end{cases} \\
        &= 
        \begin{cases}
            x + c_2 - t_{u_2} &, \text{ if } x  >  t_{u_2} \\
            c_2 &, \text{ if } x  \leq  t_{u_2}
        \end{cases}. 
    \end{align*}
    Setting $t_{u_2}=c_1$ gives  
    \begin{equation*}
        t_v(x) =
        \begin{cases}
            x + c_2 - c_1 &, \text{ if } x  >  c_1 \\
            c_2 &, \text{ if } x  \leq  c_1
        \end{cases},    
    \end{equation*}
    and we observe that for $\theta = \frac{1}{2}(c_2-c_1) - 2d$, which is a valid threshold provided that $d$ is small, \eqref{eq:propCond22} as well as $\theta + t_{u_2} = \theta + c_1 > b - 2d > b$ is satisfied for a suitable choice of $d$. Hence, $\Phi$ emulates $f$ as desired.
\end{proof}

Finally, we state the implication of this proposition for the ReLU realization, which is a direct implication of \eqref{eq:sigmaRealizing} and \eqref{eq:sigmaRealizingFiring}. 

\begin{proposition}\label{prop:ReLURealize}
    There exists a two-layer LSNN that realizes ReLU on a given bounded domain with a consistent encoding scheme.
\end{proposition}

\subsection{Realizing ReLU networks by LSNNs}
\label{sec:approx_relu_network_using_SNN}
In this section, we show that an LSNN has the capability to reproduce the output of any ReLU network. Specifically, given access to the weights and biases of an ANN, we construct an LSNN and set the parameter values based on the weights and biases of the given ANN. This leads us to the desired result.
The essential part of our proof revolves around choosing the parameters of an LSNN such that it effectively realizes the composition of an affine-linear map and the non-linearity represented by the ReLU activation. The realization of ReLU with LSNNs is proved in the previous Section \ref{sec:RealizingReLU_app}. To realize an affine-linear function using an LSNN neuron, it is necessary to ensure that the spikes from all the input neurons together result in the firing of an output neuron instead of any subset of the input neurons. We achieve that by appropriately adjusting the value of the threshold parameter. As a result, an LSNN neuron, which implements an affine-linear map, avoids partitioning of the input space.

\begin{figure}[t]
	\centering
    
    
    \begin{subfigure}[b]{0.32\textwidth}
     \centering
     \begin{tikzpicture}[node distance=3cm, scale = 1.35]
    \tikzstyle{unit}=[draw,shape=circle,minimum size=0.3cm]
    \tikzstyle{hidden}=[draw,shape=circle,fill=black!25,minimum size=0.2cm]
    \tikzstyle{hidden}=[draw,shape=circle,minimum size=0.2cm]
    \tikzstyle{auxin}=[draw = red,shape=circle,thick, minimum size=0.3cm]
    \tikzstyle{auxhd}=[draw = red,shape=circle,thick, minimum size=0.2cm]
    \tikzstyle{auxhdd}=[draw = red,shape=circle,dotted, minimum size=0.2cm]
    \tikzstyle{auxind}=[draw = red,shape=circle,dotted, minimum size=0.3cm]
    \node[unit] (input-1) at (0, -3) {};
    \node[unit] (input-2) at (1, -3) {};
    \node[auxin] (input-3) at (2, -3) {};
    \node[hidden] (hidden-1) at (1, -2) {};
    \node[auxhd] (hidden-2) at (2, -2) {};
    
    \node[hidden] (hidden-1_2) at (1, -1) {};
    \node[auxhd] (hidden-2_2) at (2, -1) {};  
      
    \node[unit] (output) at (1.5, 0) {};
    
        \draw[->] (input-1) -- (hidden-1) node[midway,above] {};
        \draw[->] (input-2) -- (hidden-1) node[midway,below] {};
        \draw[->] (input-3) -- (hidden-1) node[midway,above] {};
        \draw[->] (input-3) -- (hidden-2) node[midway,above] {};
        \draw[->] (hidden-1) -- (hidden-1_2) node[midway,above] {};
        \draw[->] (hidden-2) -- (hidden-1_2) node[midway,above] {};
        \draw[->] (hidden-2) -- (hidden-2_2) node[midway,above] {};
        \draw[->] (hidden-1_2) -- (output) node[midway,above] {};
        \draw[->] (hidden-2_2) -- (output) node[midway,above] {};

    \draw[dotted] (.7,-2.2) rectangle (2.3,0.2);
    \draw[dotted](-.2,-3.2) -- (2.4,-3.2);
    \draw[dotted](-.2,-3.2) -- (-.2,-1.8);
    \draw[dotted](-.2,-1.8) -- (1.5,-1.8);
    \draw[dotted](1.5,-1.8) -- (2.4,-3.2);
    
    \node at (0, -1.6) {\scriptsize $\mathcal{R}_f$};
    \node at (0.5, -0.2) {\scriptsize $\mathcal{R}_\sigma$};
    \node at (0, -3.5) {$t_{u_1}$};
    \node at (1, -3.5) {$t_{u_2}$};
    \node at (2, -3.5) {$t_{u_3}$};
    \node at (1.5, .5) {$t_v$};
    \draw[->, >=stealth, line width=1pt] (2.5, -1.5) -- (4.2, -1.5) node[midway, above] {\scriptsize Removal of};
    \draw[->, >=stealth, line width=1pt] (2.5, -1.5) -- (4.2, -1.5) node[midway, below] {\scriptsize auxiliary neurons};
    \end{tikzpicture}
    \caption{}
    \label{subfig:sigmafsnn}
    \end{subfigure}    
\begin{subfigure}[b]{0.33\textwidth}
     \centering
     \begin{tikzpicture}[node distance=3cm, scale = 1.35]
    \tikzstyle{unit}=[draw,shape=circle,minimum size=0.3cm]
    \tikzstyle{hidden}=[draw,shape=circle,fill=black!25,minimum size=0.2cm]
    \tikzstyle{hidden}=[draw,shape=circle,minimum size=0.2cm]
    \tikzstyle{auxin}=[draw = red,shape=circle,thick, minimum size=0.3cm]
    \tikzstyle{auxhd}=[draw = red,shape=circle,thick, minimum size=0.2cm]
    \tikzstyle{auxhdd}=[draw = red,shape=circle,dotted, minimum size=0.2cm]
    \tikzstyle{auxind}=[draw = red,shape=circle,dotted, minimum size=0.3cm]
    
    \node[unit] (input-11) at (5, -3) {};
    \node[unit] (input-12) at (6, -3) {};
    \node[auxind] (input-13) at (7, -3) {};
    
    \node[hidden] (hidden-11) at (6, -2) {};
    \node[auxhdd] (hidden-12) at (7, -2) {};
    
    \node[hidden] (hidden-11_2) at (6, -1) {};
    \node[auxhdd] (hidden-12_2) at (7, -1) {};
      
    \node[unit] (output) at (6.5, 0) {};
    
        \draw[->] (input-11) -- (hidden-11) node[midway,above] {};
        \draw[->] (input-12) -- (hidden-11) node[midway,below] {};
        \draw[->][dotted] (input-13) -- (hidden-11) node[midway,above] {};
        \draw[->][dotted] (input-13) -- (hidden-12) node[midway,above] {};
        \draw[->] (hidden-11) -- (hidden-11_2) node[midway,above] {};
        \draw[dotted][->] (hidden-12) -- (hidden-12_2) node[midway,above] {};
        \draw[->][dotted] (hidden-12) -- (hidden-11_2) node[midway,above] {};
        \draw[->] (hidden-11_2) -- (output) node[midway,above] {};
        \draw[->][dotted] (hidden-12_2) -- (output) node[midway,above] {};

    \draw[thin] (6.7,-3.3) rectangle (7.3,-0.7);
    
    \node at (5, -3.5) {$t_{u_1}$};
    \node at (6, -3.5) {$t_{u_2}$};
    \node at (7, -3.5) {$t_{u_3}$};
    \node at (6.5, .5) {$t_v$};
    \draw[->, >=stealth, line width=1pt] (7.5, -1.5) -- (9.2, -1.5) node[midway, above] {\scriptsize Parallelization};
    \end{tikzpicture}
    \caption{}
    \label{subfig:parallel_aux}
    \end{subfigure}
\begin{subfigure}[b]{0.32\textwidth}
     \centering
     \hspace{0.5cm}
     \begin{tikzpicture}[node distance=3cm, scale = 1.35]
    \tikzstyle{unit}=[draw,shape=circle,minimum size=0.3cm]
    \tikzstyle{hidden}=[draw,shape=circle,fill=black!25,minimum size=0.2cm]
    \tikzstyle{hidden}=[draw,shape=circle,minimum size=0.2cm]
    \tikzstyle{auxin}=[draw = red,shape=circle,thick, minimum size=0.3cm]
    \tikzstyle{auxhd}=[draw = red,shape=circle,thick, minimum size=0.2cm]
    \tikzstyle{auxhdd}=[draw = red,shape=circle,dotted, minimum size=0.2cm]
    \tikzstyle{auxind}=[draw = red,shape=circle,dotted, minimum size=0.3cm]
    \node[unit] (input_1) at (9, -3) {};
    \node[unit] (input_2) at (10, -3) {};
    \node[auxin] (input_3) at (11, -3) {};
    
    \node[hidden] (hidden_1) at (9, -2) {};
    \node[hidden] (hidden_2) at (10, -2) {};
    \node[auxhd] (hidden_3) at (11, -2) {};
    
    \node[hidden] (hidden_1_2) at (9, -1) {};
    \node[hidden] (hidden_2_2) at (10, -1) {};  
    \node[auxhd] (hidden_3_2) at (11, -1) {};  
      
    \node[unit] (output_p1) at (9, 0) {};
    \node[unit] (output_p2) at (10, 0) {};
    \node[auxin] (output_p3) at (11, 0) {};
    
        \draw[->] (input_1) -- (hidden_1) node[midway,above] {};
        \draw[->] (input_2) -- (hidden_1) node[midway,below] {};
        \draw[->] (input_3) -- (hidden_1) node[midway,above] {};
        \draw[->] (input_1) -- (hidden_2) node[midway,above] {};
        \draw[->] (input_2) -- (hidden_2) node[midway,below] {};
        \draw[->] (input_3) -- (hidden_2) node[midway,above] {};
        \draw[->] (input_3) -- (hidden_3) node[midway,above] {};
        \draw[->] (hidden_1) -- (hidden_1_2) node[midway,above] {};
        \draw[->] (hidden_3) -- (hidden_1_2) node[midway,above] {};
        \draw[->] (hidden_2) -- (hidden_2_2) node[midway,above] {};
        \draw[->] (hidden_3) -- (hidden_2_2) node[midway,above] {};
        \draw[->] (hidden_3) -- (hidden_3_2) node[midway,above] {};
        
        \draw[->] (hidden_1_2) -- (output_p1) node[midway,above] {};
        \draw[->] (hidden_3_2) -- (output_p1) node[midway,above] {};
        \draw[->] (hidden_2_2) -- (output_p2) node[midway,above] {};
        \draw[->] (hidden_3_2) -- (output_p2) node[midway,above] {};
        \draw[->] (hidden_3_2) -- (output_p3) node[midway,above] {};

    \node at (9, -3.5) {$t_{u_1}$};
    \node at (10, -3.5) {$t_{u_2}$};
    \node at (11, -3.5) {$t_{u_3}$};
    \node at (9, .5) {$t_{v_1}$};
    \node at (10, .5) {$t_{v_2}$};
    \end{tikzpicture}
    \caption{}
    \label{subfig:parallel}
    \end{subfigure}
    \caption{(a) Computation graph associated with an LSNN $\Phi^{\sigma \circ f}$ resulting from the concatenation of $\Phi^\sigma$ and $\Phi^f$ that realizes $\sigma(f(x_1,x_2))$, where $f$ is an affine function and $\sigma$ is the ReLU non-linearity. The auxiliary neurons are shown in red. (b) Same computation graph as in (a); when parallelizing two identical networks, the dotted auxiliary neurons can be removed and auxiliary neurons from (a) can be used for each network instead. (c) Computation graph associated with an LSNN as a result of the parallelization of two subnetworks $\Phi^{\sigma \circ f_1}$ and $\Phi^{\sigma \circ f_2}$. The auxiliary neuron in the output layer serves the same purpose as the auxiliary neuron in the input layer and is needed when concatenating two such subnetworks $\Phi^{\sigma \circ f}$.}
    \label{fig:SNNReLUConstruction}
\end{figure}
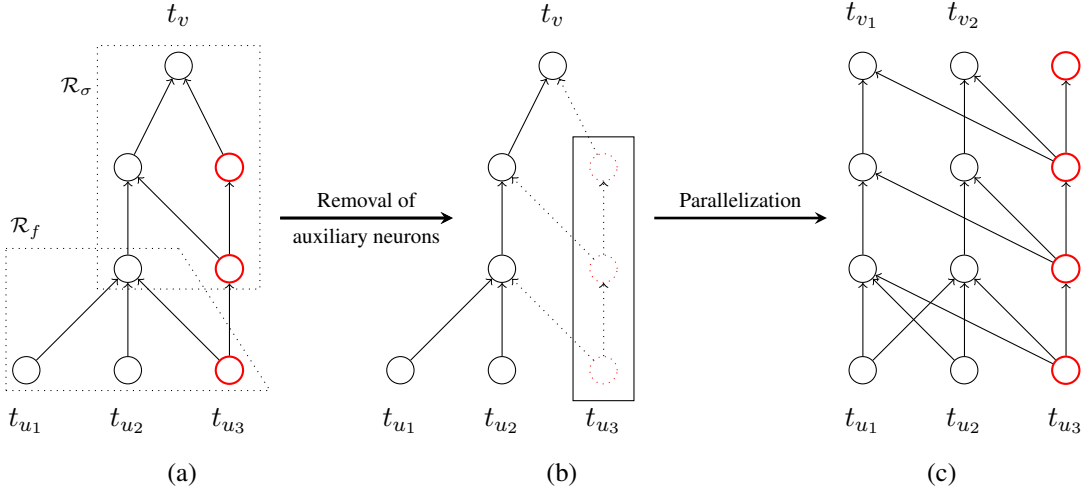

\paragraph{Setup for the proof of Theorem \ref{thm:approx_SNN_from_ANN}} 
Let $d, L \in \N$ be the width and the depth of an ANN $\Psi$, respectively, i.e.,
\begin{align*}
    \Psi = ((A^1, B^1), (A^2, B^2), \dots, (A^L, B^L)),  \text{ where } 
     &(A^\ell,B^\ell) \in \R^{d\times d}\times \R^d, 1 \leq \ell < L,\\
     &(A^L,B^L) \in \R^{d\times d} \times \R^d.
\end{align*}
For a given input domain $[a,b]^d \subset \R^d$, we denote by $\Psi^\ell = ((A^\ell, B^\ell))$ the $\ell$-th layer, where $y^0 \in [a,b]^d$ and
\begin{align}\label{eq:App_ANN}
    y^l &= \cR_{\Psi^l}(y^{l-1}) = \sigma((A^l)^T y^{l-1} + B^l), 1 \leq \ell < L, \nonumber \\
    y^L &= \cR_{\Psi^L}(y^{L-1}) = (A^L)^T y^{L-1} + B^L
\end{align}
so that $\mathcal{R}_\Psi = \mathcal{R}_{\Psi^L} \circ \cdots \circ \mathcal{R}_{\Psi^1}$. 

For the construction of the corresponding LSNN, we refer to the associated weights and delays between two LSNN neurons $u$ and $v$ by $w_{u v}$ and $d_{u v}$, respectively.

\begin{proof}[\textbf{of Theorem \ref{thm:approx_SNN_from_ANN}}]
Any multi-layer ANN $\Psi$ with ReLU activation is simply an alternating composition of affine-linear functions $(A^l)^T y^{l-1} + B^l$ and a non-linear function represented by $\sigma$. To generate the mapping realized by $\Psi$, it suffices to realize the composition of affine-linear functions and the ReLU non-linearity and then extend the construction to the whole network using concatenation and parallelization operations. We prove the result via the following steps; see also Figure \ref{fig:SNNReLUConstruction} for a depiction of the intermediate constructions.

\textbf{Step 1:} Realizing ReLU non-linearity.\\ 
Proposition \ref{prop:ReLURealize} gives the desired result. 

\textbf{Step 2:} Realizing affine-linear functions with one-dimensional range. \\
Let $f: [a,b]^d \to \R$ be an affine-linear function
\begin{equation}\label{eq:AffLin}
    f(x) = C^T x + s , \quad C^T = (c_1,\dots,c_d) \in \R^d, s \in \R.    
\end{equation}
Consider a one-layer LSNN that consists of an output neuron $v$ and d input units $u_1, \dots,u_d$. Via \eqref{eqn:firing_time_specific} the firing time of $v$ as a function of the input firing times on the linear region $R^I$ corresponding to the index set $I=\{1,\dots,d\}$ is given by
\begin{equation*}
    t_v(t_{u_1}, \dots,t_{u_d}) = \frac{\theta_v}{\sum_{i \in I} w_{u_i v}} + \frac{\sum_{i \in I} w_{u_i v}(t_{u_i}+d_{u_i v})}{\sum_{i \in I} w_{u_i v}} \quad \text{provided that }  \sum_{i \in I} w_{u_i v} >0.  
\end{equation*}
Introducing an auxiliary input neuron $u_{d+1}$ with weight $w_{u_{d+1} v} = 1 - \sum_{i \in I} w_{u_i v}$ ensures that $\sum_{i \in I \cup \{d+1\}} w_{u_i v} > 0$ and leads to the firing time 
\begin{equation*}
    t_v(t_{u_1}, \dots,t_{u_{d+1}}) = \theta_v + \sum_{i \in I \cup \{d+1\}} w_{u_i v}(t_{u_i}+d_{u_i v}) \quad \text{ on } R^{I \cup \{d+1\}}.  
\end{equation*}
Setting $w_{u_i v} = c_i$ for $i\in I$ and $d_{u_j v} = d^\prime\geq 0$ for $j\in I \cup \{d+1\}$ yields
\begin{equation*}
    t_v(t_{u_1}, \dots,t_{u_{d+1}}) = \theta_v + w_{u_{d+1} v} \cdot t_{u_{d+1}} + d^\prime + \sum_{i \in I} c_i t_{u_i} \text{ on } R^{I \cup \{d+1\}} \cap [a,b]^d.  
\end{equation*}
Therefore, an LSNN $\Phi^f =(W,D,\Theta)$ with parameters
\begin{equation*}
        W = \begin{pmatrix}
            c_1 \\
            \vdots\\
            c_{d+1} 
        \end{pmatrix},  
        D = \begin{pmatrix}
            d^\prime  \\
            \vdots\\
            d^\prime
        \end{pmatrix},
        \Theta = \theta > 0, \quad \text{ where }  c_{d+1} =  1 - \sum_{i\in I} c_i, 
\end{equation*}
and the usual encoding scheme $T_{\text{in}} \slash  T_{\text{out}} + \cdot$ and fixed firing time $t_{u_{d+1}} = t_{\text{in}}$, whereby we employ the notation from Definition \ref{definition:encoding},  realizes
\begin{align}\label{eq:App_Raf}
    \mathcal{R}_{\Phi^f} (x) &= - t_{\text{out}} +  t_v(t_{\text{in}} + x_1, \dots,t_{\text{in}} + x_d, t_{\text{in}}) =  - t_{\text{out}} +  \theta + t_{\text{in}} + d^\prime + \sum_{i \in I} c_i x_i \\
    &=  - t_{\text{out}} + \theta + t_{\text{in}} + d^\prime + f(x_1, \dots,x_d) - s \quad \text{ on } R^{I \cup \{d+1\}} \cap [a,b]^d.  
\end{align}
Choosing a large enough threshold $\theta$ ensures that a spike in $v$ is necessarily triggered after all the spikes from $u_1,\dots, u_{d+1}$ reached $v$ so that $[a,b]^d \subset R^{I \cup \{d+1\}}$ holds. It suffices to set  
\begin{equation*}
    \theta \geq \sup_{x\in [a,b]^d} \sup_{x_{\text{min}} \leq t - t_{\text{in}} - d^\prime \leq x_{\text{max}}} P_v(t), 
\end{equation*}
where $x_{\text{min}} = \min\{x_1,\dots,x_d,0\}$ and $x_{\text{max}} = \max\{x_1,\dots,x_d,0\}$, since this implies that the potential $P_v(t)$ is smaller than the threshold to trigger a spike in $v$ on the time interval associated to feasible input spikes, i.e., $v$ emits a spike after the last spike from an input neuron arrived at $v$. Applying \eqref{eqn:firing_time} shows that for $x \in [a,b]^d$ and $t \in [x_{\text{min}} + t_{\text{in}} + d^\prime, x_{\text{max}} + t_{\text{in}} + d^\prime]$
\begin{align*}
    P_v(t) &= \sum_{i \in I} w_{u_i v} (t - (t_{\text{in}} + x_i) - d_{u_i v}) + w_{u_{d+1}v}(t - t_{\text{in}}- d_{u_{d+1} v}) =  t - d^\prime - t_{\text{in}} + \sum_{i \in I} c_i x_i \\
    &\leq x_{\text{max}} + d\norm{C}_{\infty} \norm{x}_{\infty} \leq (1 + d\norm{C}_{\infty}) \max\{|a|, |b|\}.
\end{align*}
Hence, we set
\begin{equation*}
    \theta = (1 + d\norm{C}_{\infty}) \max\{|a|, |b|\} + s + |s| \quad \text{ and } \quad  t_{\text{out}} = \theta - s + t_{\text{in}} + d^\prime
\end{equation*}
to obtain via \eqref{eq:App_Raf} that
\begin{equation}\label{eq:App_Explain}
    \mathcal{R}_{\Phi^f} (x) = - t_{\text{out}} +  t_v(t_{\text{in}} + x_1, \dots,t_{\text{in}} + x_d, t_{\text{in}}) =  f(x) \quad \text{ for } x \in [a,b]^d.     
\end{equation}
Note that the reference time $t_{\text{out}} = (1 + d\norm{C}_{\infty}) \max\{|a|, |b|\} + |s| + t_{\text{in}} + d^\prime$ is independent of the specific parameters of $f$ in the sense that only upper bounds $\norm{C}_{\infty}, |s|$ on the parameters are relevant. Therefore, $t_{\text{out}}$ (with the associated choice of $\theta$) can be applied for different affine linear functions as long as the upper bounds remain valid. This is necessary for the composition and parallelization of subnetworks in the subsequent construction.

\textbf{Step 3:} Realizing compositions of affine-linear functions with one-dimensional range and ReLU. \\
The next step is to realize the composition of ReLU $\sigma$ with an affine linear mapping $f$ defined in \eqref{eq:AffLin}. To that end, we want to concatenate the networks $\Phi^\sigma$ and $\Phi^f$ constructed in Step 1 and Step 2, respectively, via Lemma \ref{lemma:concatenation}. To employ the concatenation operation we need to perform the following steps:
\begin{enumerate}
    \item Find an appropriate input domain $[a^\prime, b^\prime] \subset \R$, that contains the image $f([a,b]^d)$ so that parameters and reference times of $\Phi^\sigma$ can be fixed appropriately (see Proposition \ref{prop:ReLURealize} for the detailed conditions on how to choose the parameter).
    \item Ensure that the output reference time $t^f_{\text{out}}$ of $\Phi^f$ equals the input reference time $t^\sigma_{\text{in}}$ of $\Phi^\sigma$.
    \item Ensure that the number of neurons in the output layer of $\Phi^f$ is the same as the number of input neurons in $\Phi^\sigma$.
\end{enumerate}
For the first point, note that
\begin{equation*}
    |f(x)| = |C^T x + s| \leq d\norm{C}_\infty \cdot \norm{x}_\infty + |s| \leq d\norm{C}_\infty \cdot \max\{|a|, |b|\} + |s| \text{ for all } x \in [a,b]^d.
\end{equation*}
Hence, we can use the input domain 
\begin{equation*}
    [a^\prime, b^\prime] = [- d\norm{C}_\infty \cdot \max\{|a|, |b|\} + |s|, d\norm{C}_\infty \cdot \max\{|a|, |b|\} + |s|]    
\end{equation*} 
and specify the parameters of $\Phi^\sigma$ accordingly. Additionally, recall from Proposition \ref{prop:ReLURealize} that $t^\sigma_{\text{in}}$ can be chosen freely, so we may fix $t^\sigma_{\text{in}} = t^f_{\text{out}}$, where $t^f_{\text{out}}$ is established in Step 2. It remains to consider the third point. To realize ReLU, an additional auxiliary neuron in the input layer of $\Phi^\sigma$ with constant input $t^\sigma_{\text{in}}$ was introduced. Hence, we also need to add an auxiliary output neuron in $\Phi^f$ with (constant) firing time $t^f_{\text{out}}= t^\sigma_{\text{in}}$ so that the corresponding output and input dimension and their specification match. This is achieved by introducing a single synapse from the auxiliary neuron in the input layer of $\Phi^f$ to the newly added output neuron and by specifying the parameters of the newly introduced synapse and neuron suitably. Formally, the adapted network $\Phi^f=(W,D,\Theta)$ is given by
\begin{equation*}
        W = \begin{pmatrix}
            c_1 & 0\\
            \vdots & \vdots \\
            c_d & 0 \\
            c_{d+1} & 1 
        \end{pmatrix},  
        D = \begin{pmatrix}
            d^\prime & 0 \\
            \vdots & \vdots\\
            d^\prime & 0 \\
            d^\prime & d^\prime 
        \end{pmatrix},
        \Theta = \begin{pmatrix}
            \theta \\
            t^f_{\text{out}} - t^f_{\text{in}} - d^\prime 
        \end{pmatrix}, 
\end{equation*}
where the values of the parameters are specified in Step 2.

Then the realization of the concatenated network $\Phi^{\sigma\circ f}$ is the composition of the individual realizations. This is exemplarily demonstrated in Figure \ref{subfig:sigmafsnn} for the two-dimensional input case. 
By analyzing $\Phi^{\sigma\circ f}$, we conclude that a three-layer LSNN with 
\begin{equation*}
    N(\Phi^{\sigma\circ f}) = N(\Phi^{\sigma}) - N_{0}(\Phi^{\sigma}) + N(\Phi^{f})  =  5 - 2 + d + 3 = d+6
\end{equation*}
computational units can realize $\sigma \circ f$ on $[a,b]^d$, where $N_{0}(\Phi^{\sigma})$ denotes the number of neurons in the input layer of $\Phi^{\sigma}$. 

\textbf{Step 4:} Realizing layer-wise computation of $\Psi$. \\ 
The computations performed in a layer $\Psi^\ell$ of $\Psi$ are described in \eqref{eq:ANNcomp}. Hence, for $1\leq \ell < L$ the computation can be expressed as
\begin{equation*}
    \mathcal{R}_{\Psi^\ell}(y^{l-1})=  \sigma((A^l)^T y^{l-1} + B^l) = \begin{pmatrix}
                                    \sigma(\sum_{i=1}^d (A_{1,i}^l)^T y_i^{l-1} + B_1^l) \\
                                    \vdots \\
                                    \sigma(\sum_{i=1}^d (A_{d,i}^l)^T y_i^{l-1} + B_d^l)
                                \end{pmatrix}    
                                =: \begin{pmatrix}
                                    \sigma(f_1(y^{l-1})) \\
                                    \vdots \\
                                    \sigma(f_d(y^{l-1}))
                                \end{pmatrix}, 
\end{equation*}
where $f^\ell_1,\dots, f^\ell_d$ are affine linear functions with one-dimensional range on the same input domain $[a^{\ell-1}, b^{\ell-1}]^d \subset \R^d$, where $[a^0,b^0] = [a,b]$ and $[a^{\ell}, b^{\ell}]$ is the range of 
\begin{equation*}
    (\sigma \circ f^{\ell-1}_1,\dots, \sigma\circ f^{\ell-1}_d)([a^{\ell-1}, b^{\ell-1}]^d) .
\end{equation*}
Thus, via Step 3, we construct LSNNs $\Phi^\ell_1,\dots, \Phi^\ell_d$ that realize $\sigma \circ f^{\ell}_1,\dots, \sigma\circ f^{\ell}_d$ on $[a^{\ell-1}, b^{\ell-1}]$. Note that by choosing appropriate parameters in the construction performed in Step 2 (as described below \eqref{eq:App_Explain}), e.g., $\norm{A^l}_\infty$ and $\norm{B^l}_\infty$, we can employ the same input and output reference time for each $\Phi^\ell_1,\dots, \Phi^\ell_d$. Consequently, we can parallelize $\Phi^\ell_1,\dots, \Phi^\ell_d$ (see Lemma \ref{lemma:parallelization}) and obtain networks $\Phi^\ell = P(\Phi^\ell_1,\dots, \Phi^\ell_d)$ realizing $\mathcal{R}_{\Psi^\ell}$ on $[a^{\ell-1}, b^{\ell-1}]^d$. Finally, $\Psi^L$ can be directly realized via Step 2 by an LSNN $\Phi^L$ (as in the last layer no activation function is applied and the output is d-dimensional).   
Although $\Phi^\ell$ already performs the desired task of realizing $\mathcal{R}_{\Psi^\ell}$ we can slightly simplify the network. By construction in Step 3, each $\Phi^\ell_i$ contains two auxiliary neurons in the hidden layers. Since the input and output reference time is chosen consistently for $\Phi^\ell_1, \dots, \Phi^\ell_d$, we observe that the auxiliary neurons in each $\Phi^\ell_i$ perform the same operations and have the same firing times. Therefore, without changing the realization of $\Phi^\ell$ we can remove the auxiliary neurons in $\Phi^\ell_2, \dots, \Phi^\ell_d$ and introduce synapses from the auxiliary neurons in $\Phi^\ell_1$ accordingly. This is exemplarily demonstrated in Figure \ref{subfig:parallel_aux} for the case $d=2$. After this modification, we observe that $L(\Phi^\ell) = L(\Phi^\ell_i) = 3$ and 
\begin{align*}
    N(\Phi^\ell) &= N(\Phi^\ell_1) + \sum_{i=2}^d \big( N(\Phi^\ell_i) -2 - N_0(\Phi^\ell_i) \big) =  d  N(\Phi^\ell_1) - (d-1)(2 + N_0(\Phi^\ell_1))\\
    &= d(d+6) - 2(d-1) - (d-1)(d+1) = 4d + 3 \quad \text{ for } 1\leq \ell <L,
\end{align*}
whereas $L(\Phi^L) = 1$ and $N(\Phi^L) = 2d+1$.  

\textbf{Step 5:} Realizing compositions of layer-wise computations of $\Psi$.\\  
The last step is to compose the realizations $\mathcal{R}_{\Phi^1},\dots, \mathcal{R}_{\Phi^L}$ to obtain the realization 
\begin{equation*}
    \mathcal{R}_{\Phi^L} \circ \cdots \circ \mathcal{R}_{\Phi^1} = \mathcal{R}_{\Psi^L} \circ \cdots \circ \mathcal{R}_{\Psi^1} = \mathcal{R}_{\Psi}.
\end{equation*}
As in Step 3, it suffices again to verify that the concatenation of the networks $\mathcal{R}_{\Phi^1},\dots, \mathcal{R}_{\Phi^L}$ is feasible. First, note that for $\ell=1,\dots, L$ the input domain of $\mathcal{R}_{\Phi^\ell}$ is given by $[a^{\ell-1}, b^{\ell-1}]^d$ so that, we can fix the suitable output reference time $T^{\Phi^\ell}_{\text{out}}=t^{\Phi^\ell}_{\text{out}}\, (1,\dots,1)^T \in \R^d$ based on the parameters of the network, the domain $[a^{\ell-1}, b^{\ell-1}]$, and some input reference time $T^{\Phi^\ell}_{\text{in}} =t^{\Phi^\ell}_{\text{in}}\, (1,\dots,1)^T \in \R^d$. By construction in Steps 2 - 4 $T^{\Phi^\ell}_{\text{in}}$ can be chosen freely. Hence setting $T^{\Phi^{\ell+1}}_{\text{in}} = T^{\Phi^\ell}_{\text{out}}$ ensures that the reference times of the corresponding networks agree. It is left to align the input dimension of $\Phi^{\ell+1}$ and the output dimension of $\Phi^{\ell}$ for $\ell=1,\dots, L-1$. Due to the auxiliary neuron in the input layer of $\Phi^{\ell+1}$, we also need to introduce an auxiliary neuron in the output layer of $\Phi^{\ell}$ (see Figure \ref{subfig:parallel}) with the required firing time $t^{\Phi^{\ell+1}}_{\text{in}} = t^{\Phi^\ell}_{\text{out}}$. Similarly, as in Step 3, it suffices to add a single synapse from the auxiliary neuron in the previous layer to obtain the desired firing time.   

Thus, we conclude that $\Phi = \Phi^L \bullet \cdots \bullet \Phi^1$ realizes $\mathcal{R}_\Psi$ on $[a,b]$, as desired. The complexity of $\Phi$ in the number of layers and neurons is given by
\begin{equation*}
    L(\Phi) = \sum_{\ell=1}^L L(\Phi^\ell) = 3 L -2 = 3 L (\Psi) - 2
\end{equation*}
and
\begin{align*}
    N(\Phi) &= N(\Phi^1) + \sum_{\ell=2}^L \big(N(\Phi^\ell) - N_0(\Phi^\ell)\big) + (L-1) \\
    &= 4d + 3 + (L-2)(4d + 3 - (d+1)) + (2d+1 - (d+1)) + (L-1) \\
    &= 3L(d+1) -(2d+1) \\
    &= N(\Psi) + L(2d +3) - (2d+2)
\end{align*}
\end{proof}

\begin{remark}
Note that the delays play no significant role in the proof of the above theorem. Nevertheless, they can be employed to alter the timing of spikes, consequently impacting the firing time and the resulting output. However, the exact function of delays requires further investigation.
The primary objective is to present a construction that proves the existence of an LSNN capable of accurately reproducing the output of any ReLU network. 
\end{remark}

\end{document}